\documentclass[twoside]{article}
\usepackage{amsmath,amssymb,amsthm,amsfonts,mathtools,bm}
\usepackage[preprint]{aistats2027}
\usepackage{url}
\usepackage{algorithm}
\usepackage[noend]{algpseudocode}
\usepackage{booktabs}
\usepackage{graphicx}
\usepackage{xcolor}
\usepackage{tikz}
\usetikzlibrary{arrows,positioning}
\newcommand{\defeq}{\coloneqq}
\definecolor{figNavy}{HTML}{243B53}
\definecolor{figBlue}{HTML}{0072B2}
\definecolor{figOrange}{HTML}{D55E00}
\definecolor{figGreen}{HTML}{009E73}
\definecolor{figPurple}{HTML}{6F4C9B}
\definecolor{figGold}{HTML}{D89000}
\definecolor{figTeal}{HTML}{008A99}

\newcommand{\ch}[1]{#1}

\newif\ifreviewadditions
\reviewadditionstrue

\newcommand{\methodname}{DI-SCUSUM}

\newcommand{\methodtag}{\mathrm{DI}}
\newtheorem{theorem}{Theorem}
\newtheorem{lemma}{Lemma}
\newtheorem{proposition}{Proposition}

\newtheorem{remark}{Remark}

\begin{document}

%
\runningtitle{Quickest Change Detection with Diffusion-Integrated Scores}

%
\runningauthor{Adibi, Maleki, Kulkarni, and Poor}

\twocolumn[

\aistatstitle{Quickest Change Detection with Diffusion-Integrated Scores}

\aistatsauthor{Arman Adibi$^{*}$ \And Mohammadreza Maleki$^{*}$ \And
Sanjeev Kulkarni \And H. Vincent Poor}

\aistatsaddress{Augusta University \And Toronto Metropolitan University \And
Princeton University \And Princeton University} ]

\begingroup
\renewcommand{\thefootnote}{*}
\footnotetext{These authors contributed equally to this work.}
\endgroup

\begin{abstract}
Classical CUSUM relies on the log-likelihood ratio of the underlying distributions, which cannot generally be computed from finite pre- and post-change samples alone.
We propose diffusion-integrated score CUSUM (DI-SCUSUM), a training-free detector.
We add Gaussian noise to the samples to form two smooth density estimates and calculate their Hyv\"arinen scores exactly, without training a score network.
For each incoming observation, we sample a diffusion time, perturb the observation, and use the importance-weighted score difference as an increment in the DI-SCUSUM recursion.
Under the assumption that observations follow the fixed empirical distributions, the post-change mean increment is proportional to the Kullback--Leibler (KL) divergence from the smoothed post-change to the smoothed pre-change empirical distribution.
We establish exponential false-alarm scaling and a first-order delay bound that, for a fixed threshold and increment scaling, is inversely proportional to the KL divergence.
In the calibrated anisotropic Gaussian simulation, DI-SCUSUM nearly matches likelihood-ratio CUSUM and reduces the measured detection delay by about 91\% relative to score-based CUSUM.
On MNIST and Oxford-IIIT Pet, DI-SCUSUM also has lower empirical conditional detection delay than SCUSUM at comparable false-alarm levels.
\end{abstract}

\section{Introduction}

Quickest change detection (QCD) seeks to identify an unknown distributional
change with minimal delay subject to false-alarm control. When the pre- and
post-change densities are known, CUSUM
\cite{page1954continuous,poor2008quickest} accumulates log-likelihood-ratio evidence and
provides the classical benchmark. We consider the sample-based setting, where finite samples from each distribution are collected separately from the monitored stream. The underlying
log-likelihood ratio is unavailable, so a sequential procedure must extract
computable evidence of a distributional change from each new observation.
In a CUSUM recursion, the evidence should have a negative mean before the change to prevent sustained accumulation and a positive mean afterward to drive the statistic toward the alarm threshold.

\definecolor{figInk}{HTML}{263746}
\definecolor{figBlue}{HTML}{2F6B9A}
\definecolor{figOrange}{HTML}{C96B2C}
\definecolor{figTeal}{HTML}{238A7E}
\definecolor{figPurple}{HTML}{765A9D}
\definecolor{figGold}{HTML}{B88720}
\definecolor{figPanel}{HTML}{F7F9FB}

\begin{figure*}[t]
\centering
\resizebox{0.9\textwidth}{!}{%
\begin{tikzpicture}[
panel/.style={
  draw=figInk!28, fill=figPanel, rounded corners=3pt,
  line width=0.65pt
},
heading/.style={
  font=\footnotesize\bfseries, text=figInk
},
section/.style={
  font=\scriptsize\bfseries, text=figInk
},
label/.style={
  font=\scriptsize, text=figInk, align=center
},
axis/.style={
  ->, >=stealth, draw=figInk!62, line width=0.6pt
},
preline/.style={
  draw=figBlue, line width=1.05pt
},
postline/.style={
  draw=figOrange, line width=1.05pt
},
statline/.style={
  draw=figTeal!90!black, line width=1.1pt
},
change/.style={
  draw=figOrange, densely dashed, line width=0.75pt
},
alarm/.style={
  draw=figTeal!90!black, densely dashed, line width=0.75pt
},
threshold/.style={
  draw=figInk!65, densely dashed, line width=0.8pt
},
delay/.style={
  <->, >=stealth, draw=figInk!75, line width=0.8pt
},
samplemodel/.style={
  draw=figInk!25, fill=white, rounded corners=2pt,
  line width=0.65pt, minimum height=0.72cm,
  text width=3.05cm, font=\scriptsize, align=center
},
stage/.style={
  draw=figInk!32, fill=white, rounded corners=2.5pt,
  line width=0.7pt, minimum height=2.12cm,
  minimum width=1.72cm, inner sep=0pt
},
flow/.style={
  ->, >=stealth, draw=figInk!85, line width=0.85pt,
  shorten <=0.5pt, shorten >=0.5pt
},
drift/.style={
  draw=figInk!20, fill=white, rounded corners=2pt,
  line width=0.6pt, minimum height=0.48cm,
  text width=7.45cm, font=\scriptsize,
  text=figInk, align=center
}]

\path[panel] (0,0) rectangle (8.00,5.45);
\path[panel] (8.25,0) rectangle (17.40,5.45);

\node[heading, anchor=west] at (0.30,5.12)
  {(a) Change and detection};
\node[heading, anchor=west] at (8.55,5.12)
  {(b) DI-SCUSUM construction};

\begin{scope}[shift={(0.72,3.02)}]
  \path[fill=figBlue!5]
    (0,-0.05) rectangle (3.15,1.42);
  \path[fill=figOrange!6]
    (3.15,-0.05) rectangle (6.55,1.42);

  \draw[axis] (0,0) -- (6.58,0);
  \draw[axis] (0,-0.04) -- (0,1.45)
    node[above] {\(X_n\)};
  \node[label, anchor=north east] at (6.52,-0.04)
    {time \(n\)};

  \draw[change] (3.15,-0.02) -- (3.15,1.38);
  \node[label, text=figInk, anchor=south]
    at (3.15,1.38) {change \(\nu\)};

  \draw[preline] plot coordinates {
    (0.10,.25) (.42,.43) (.74,.20) (1.06,.38)
    (1.38,.18) (1.70,.42) (2.02,.23) (2.34,.45)
    (2.66,.21) (2.98,.36) (3.12,.27)
  };

  \draw[postline] plot coordinates {
    (3.12,.27) (3.34,.87) (3.66,1.04) (3.98,.82)
    (4.30,1.10) (4.62,.90) (4.94,1.17) (5.26,.89)
    (5.58,1.08) (5.90,.93) (6.22,1.13)
  };

  \node[label, text=figInk] at (1.55,.73)
    {pre-change \(P_\infty\)};
  \node[label, text=figInk] at (4.90,.43)
    {post-change \(P_1\)};
\end{scope}

\begin{scope}[shift={(0.72,.52)}]
  \draw[axis] (0,0) -- (6.58,0);
  \draw[axis] (0,-0.04) -- (0,1.62)
    node[above] {\(Z_n\)};
  \node[label, anchor=north east] at (6.52,-0.04)
    {time \(n\)};

  \draw[threshold] (0,1.02) -- (6.55,1.02);
  \node[label, text=figInk, anchor=south west]
    at (0.15,1.02) {threshold \(\tau\)};

  \draw[change] (3.15,-0.02) -- (3.15,1.52);
  \draw[alarm] (4.72,-0.02) -- (4.72,1.52);

  \draw[statline] plot coordinates {
    (0.06,.04) (.38,.20) (.70,.03) (1.02,.17)
    (1.34,.04) (1.66,.25) (1.98,.03) (2.30,.19)
    (2.62,.05) (2.94,.20) (3.12,.06) (3.38,.23)
    (3.64,.40) (3.90,.55) (4.16,.73) (4.42,.91)
    (4.72,1.14) (4.98,1.31) (5.24,1.24) (5.50,1.42)
    (5.76,1.34) (6.04,1.50) (6.32,1.42)
  };

  \fill[figTeal!90!black] (4.72,1.14) circle (2.2pt);

  \node[label, text=figInk, anchor=north]
    at (3.15,-0.03) {\(\nu\)};
  \node[label, text=figInk, anchor=north]
    at (4.72,-0.03) {\(T_{\mathrm{DI}}\)};

  \draw[delay] (3.15,1.42) -- (4.72,1.42)
    node[label, text=figInk, midway, above=2pt]
    {post-change samples \(T_{\mathrm{DI}}-\nu+1\)};
\end{scope}

\node[section, anchor=west] at (8.62,4.63)
  {Offline sample-based models};

\node[
  samplemodel,
  draw=figBlue!70,
  text=figInk
] (premodel) at (10.63,4.05)
  {\textbf{Pre-change}\\[-1pt]
   \(\mathcal U\longmapsto
   \widehat p_{\infty,t}^{\varepsilon}\)};

\node[
  samplemodel,
  draw=figOrange!72,
  text=figInk
] (postmodel) at (15.02,4.05)
  {\textbf{Post-change}\\[-1pt]
   \(\mathcal V\longmapsto
   \widehat p_{1,t}^{\varepsilon}\)};

\draw[figInk!18, line width=0.55pt]
  (8.62,3.48) -- (17.05,3.48);

\node[section, anchor=west] at (8.62,3.27)
  {Online update for observation \(X_n\)};

\node[stage] (diffuse) at (9.59,2.02) {};
\node[stage] (score) at (11.82,2.02) {};
\node[stage] (weight) at (14.05,2.02) {};
\node[stage] (update) at (16.28,2.02) {};

\foreach \stagebox/\stagetitle in {
  diffuse/{1. Diffuse},score/{2. Score},
  weight/{3. Weight},update/{4. Update}} {
  \node[font=\scriptsize\bfseries, text=figInk, anchor=north,
        inner sep=0pt] at ([yshift=-0.15cm]\stagebox.north)
    {\stagetitle};
}

\node[label, text width=1.50cm, inner sep=0pt] at (9.59,1.84)
  {draw \(t_n,W_n\);\\
   perturb \(X_n\)\\
   to form \(Y_n\)};
\node[label, text width=1.50cm, inner sep=0pt] at (11.82,1.84)
  {evaluate score differences\\
   \(\Delta_{t_n}(Y_n)\)};
\node[label, text width=1.50cm, inner sep=0pt] at (14.05,1.84)
  {multiply by\\
   \(2\lambda/f(t_n)\)\\
   to obtain\\
   \(z_n^{\mathrm{DI}}\)};
\node[label, text width=1.50cm, inner sep=0pt] at (16.28,1.84)
  {reflect \(Z_n\);\\
   alarm when\\
   \(Z_n\geq\tau\), yielding\\
   \(T_{\mathrm{DI}}\)};

\draw[flow] (diffuse.east) -- (score.west);
\draw[flow] (score.east) -- (weight.west);
\draw[flow] (weight.east) -- (update.west);

\node[drift] at (12.82,.52)
  {\textcolor{figInk}{
     pre-change:
     \(\mathbb E_\infty[z_n^{\mathrm{DI}}]<0\)}
   \qquad
   \textcolor{figInk}{
     post-change:
     \(\mathbb E_1[z_n^{\mathrm{DI}}]>0\)}};

\end{tikzpicture}%
}
\caption{\textbf{DI-SCUSUM.}
(a) The change begins at $X_\nu$; delay counts post-change observations through the alarm. Accumulated evidence crosses $\tau$ and
triggers an alarm. (b) Pre- and post-change samples define Gaussian-smoothed models. Each new observation is perturbed at a sampled diffusion time; the resulting score difference is importance-weighted and accumulated by CUSUM to detect a distributional change.}
\label{fig:cfdiscusum-overview}
\end{figure*}
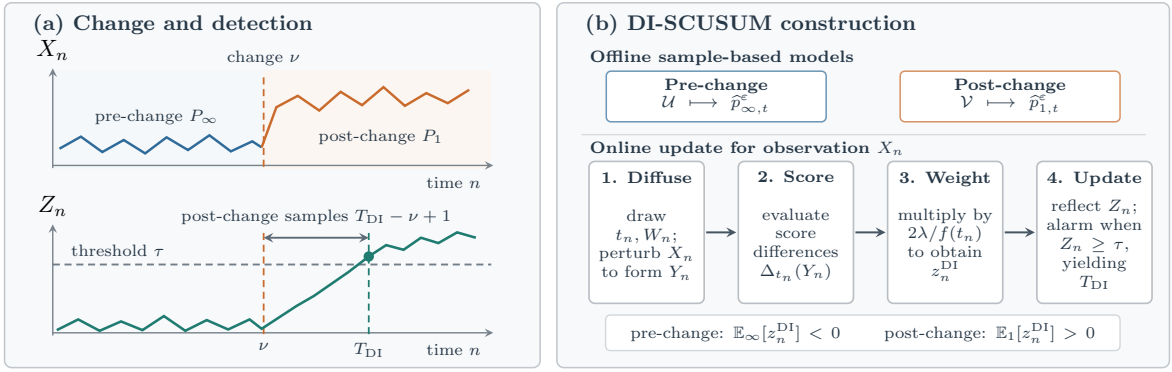

\paragraph{Related Work.}
Likelihood-ratio CUSUM is the classical QCD benchmark, with
false-alarm and delay guarantees for known distributional models
\cite{page1954continuous,lorden1971procedures,moustakides1986optimal,
pollak1985optimal,poor2008quickest,lai2011quickestsearch}.
When likelihoods are unavailable, LPA-CUSUM estimates log-partition
ratios for unnormalized models \cite{adibi2025lpa}, while the Scan
$B$-statistic and CALM use kernel comparisons to detect changes from
samples \cite{li2019scanb,cobb2022calm}.
SCUSUM replaces log-likelihood ratios with normalization-invariant
Hyv\"arinen-score differences \cite{hyvarinen2005score,wu2023scusum};
its mean evidence depends on single-scale Fisher divergence. Denoising score matching has also been used to learn score functions for CUSUM-based change detection \cite{zhou2025denoising}.
Matrix-weighted score methods modify the score geometry for hypothesis
testing and change-point detection \cite{moushegian2025diffusion}.
Diffusion models instead provide scores across noise levels
\cite{vincent2011denoising,ho2020ddpm,karras2022elucidating}, and
Gaussian-smoothed empirical distributions admit closed-form scores
\cite{scarvelis2025closedform}.
A central question is how to integrate score evidence across Gaussian smoothing scales so that the resulting mean increments recover directed KL divergences and support sequential change-detection guarantees.

We therefore propose diffusion-integrated score CUSUM (DI-SCUSUM), a training-free procedure that combines score evidence across Gaussian smoothing scales to detect distributional changes using pre- and post-change samples. Gaussian smoothing converts the pre- and post-change samples into two smooth densities with closed-form Hyv\"arinen scores. For each incoming observation, we sample a diffusion time, add Gaussian noise at the corresponding scale, and compute the score difference. We weight the score difference by the inverse of the diffusion-time sampling density and accumulate the resulting increments using CUSUM. When observations follow distinct fixed empirical distributions, the relative de Bruijn identity \cite{stam1959information,valerotoranzo2018debruijn} links the mean increments to directed KL divergences between their Gaussian-smoothed versions, giving a negative mean before the change and a positive mean afterward. This connection supports the false-alarm and detection-delay guarantees.

\paragraph{Contributions.}
Our contributions are:
(1) A sample-based procedure for sequential change detection using
closed-form Hyv\"arinen scores and an unbiased importance-sampling
estimate for the diffusion-time integral of their difference.
(2) Directed KL identities for the mean increments under fixed empirical
plug-in models, with exponential false-alarm scaling and first-order delay
bounds under the stated assumptions.
(3) A finite-window information identity and an anisotropic Gaussian
analysis showing when integration recovers information lost by
single-scale SCUSUM. (4) A diffusion-time sampling density that minimizes the post-change increment's second moment, with Monte Carlo, quadrature, and Rao--Blackwellized implementations \cite{rao1945information,blackwell1947conditional}, and an evaluation of the quadrature implementation on Gaussian and image data.

\section{Methodology}
\label{sec:methodology}
QCD monitors a data stream for an unknown transition
from a pre-change to a post-change regime and seeks to raise an alarm quickly
subject to false-alarm control. Figure~\ref{fig:cfdiscusum-overview}(a)
illustrates this objective: the change occurs at \(\nu\), the nonnegative
statistic \(Z_n\) crosses the threshold at the alarm time \(T\), and the
delay counts the observations from $X_\nu$ through $X_T$, inclusive: $T-\nu+1$. Increasing the threshold reduces
false alarms but generally increases detection delay.

Figure~\ref{fig:cfdiscusum-overview}(b) shows how DI-SCUSUM turns
each observation into evidence of a distributional change. We first collect
pre- and post-change samples separately from the monitored stream and
Gaussian-smooth them to obtain two density estimates with closed-form
Hyv\"arinen scores. For each new observation, we sample a diffusion time,
add Gaussian noise at the corresponding scale, and compare the two scores.
We weight the score difference to account for the sampled time and accumulate
it using the reflected CUSUM recursion. Negative mean evidence before the
change opposes accumulation, while positive mean evidence afterward drives
the statistic toward the alarm threshold.

We keep the pre- and post-change sample sets fixed throughout the
analysis; all theoretical expectations are taken conditional on these sets.

\subsection{Change Model and Gaussian Smoothing}
\label{sec:score-evidence}

Let \(X_1,X_2,\ldots\in\mathbb R^d\) be independent observations and let \(\mathcal F_n=\sigma(X_k,t_k,W_k:1\leq k\leq n)\)
include the observations and the independent auxiliary draws used by
Algorithm~\ref{alg:cf-discusum}.
Let \(\nu\geq1\) index the first post-change observation (with \(\nu=1\) denoting an immediate change). We assume
\[
X_1,\ldots,X_{\nu-1}\stackrel{\mathrm{i.i.d.}}{\sim}P_\infty,
\qquad
X_\nu,X_{\nu+1},\ldots\stackrel{\mathrm{i.i.d.}}{\sim}P_1,
\]
where \(P_\infty\) and \(P_1\) denote the pre- and post-change
probability distributions.
Expectations under no change and a change at \(\nu\) are denoted by
\(\mathbb E_\infty\) and \(\mathbb E_\nu\). The alarm time \(T\) is a stopping time with respect to
\(\mathcal F_n\): the decision uses only observations and auxiliary draws
available by time \(n\).
We measure the false alarms by the
average run length
\[
\operatorname{ARL}(T)=\mathbb E_\infty[T]
\]
and measure the worst-case conditional mean delay by Pollak's criterion
\cite{pollak1985optimal}
\[
\operatorname{CADD}(T)
=\sup_{\nu\geq1}
\mathbb E_\nu[T-\nu+1\mid T\geq\nu].
\]
Given a target \(\gamma>1\), we minimize CADD subject to
\(\operatorname{ARL}(T)\geq\gamma\), requiring an average of at least
\(\gamma\) observations before an alarm under no change. Larger
\(\gamma\) therefore requires less frequent false alarms
\cite{page1954continuous,lorden1971procedures,moustakides1986optimal}.

Let
\(\mathcal U=\{u_i\}_{i=1}^{N_0}\) and
\(\mathcal V=\{v_j\}_{j=1}^{N_1}\) be the pre- and post-change reference
sets, of sizes \(N_0=|\mathcal U|\) and \(N_1=|\mathcal V|\), respectively.
They define empirical measures
\[
\widehat P^0=\frac1{N_0}\sum_{i=1}^{N_0}\delta_{u_i},
\qquad
\widehat Q^0=\frac1{N_1}\sum_{j=1}^{N_1}\delta_{v_j},
\]
where \(\delta_x(A)=\mathbf 1_{\{x\in A\}}\) for every measurable
set \(A\subseteq\mathbb R^d\), and \(x\) is an observed sample \(u_i\) or \(v_j\). For the theoretical plug-in analysis, we assume \(P_\infty=\widehat P^0\) and \(P_1=\widehat Q^0\), so observations follow the fixed empirical distributions. For probability
distributions \(P\) and \(Q\) with densities \(p\) and \(q\), we use the
standard definition
\[
\operatorname{KL}(P\|Q)
=\mathbb E_{X\sim P}\!\left[\log\frac{p(X)}{q(X)}\right].
\]
KL divergence is infinite if \(P\) assigns positive probability
to a set that \(Q\) assigns zero probability. Independent empirical measures
from continuous distributions almost surely have disjoint supports, making
both directed KL divergences infinite even without an underlying change.

To give the two probability distributions positive common support, fix a base scale
\(\varepsilon>0\) and add independent Gaussian noise:
\begin{align*}
\widehat P^\varepsilon
&=\widehat P^0*\mathcal N(0,2\varepsilon I_d),&
\widehat Q^\varepsilon
&=\widehat Q^0*\mathcal N(0,2\varepsilon I_d),
\end{align*}
where \(*\) denotes convolution and \(I_d\) is the \(d\times d\) identity.
Each convolution is a finite Gaussian mixture and therefore has a
positive smooth density on \(\mathbb R^d\); its Gaussian tails also make both
directed KL divergences finite.

For an additional diffusion time \(t\geq0\), convolve both
probability distributions with
\(\mathcal N(0,2tI_d)\). Gaussian covariances add, giving the total scale
\(h_t=\varepsilon+t\) and hence
\begin{equation}
\begin{aligned}
\widehat P_{\infty,t}^{\varepsilon}
&=\widehat P^\varepsilon*\mathcal N(0,2tI_d)
=\widehat P^0*\mathcal N(0,2h_tI_d),\\
\widehat P_{1,t}^{\varepsilon}
&=\widehat Q^\varepsilon*\mathcal N(0,2tI_d)
=\widehat Q^0*\mathcal N(0,2h_tI_d).
\end{aligned}
\label{eq:diffused-empirical-distributions}
\end{equation}
Both \(\varepsilon\) and \(t\) are Gaussian smoothing parameters. The fixed \(\varepsilon>0\) makes the empirical
models smooth and their KL divergences finite. Varying \(t\geq0\) then lets
us compare them across smoothing scales. Using \(\varepsilon\) alone would
give only one scale; integrating over \(t\) connects the score evidence to
KL divergence. Appendix~\ref{app:parameter-roles} gives further details.

Because the Gaussian kernel is positive and infinitely differentiable, each
density \(\widehat p_{r,t}^{\varepsilon}\),
where \(r=\infty\) denotes the pre-change model and \(r=1\)
the post-change model, in
\eqref{eq:diffused-empirical-distributions} is positive and smooth. Moreover,
the Gaussian kernel, and hence each finite mixture, satisfies the classical
heat equation
\[
\partial_t\widehat p_{r,t}^{\varepsilon}(y)
=\Delta_y\widehat p_{r,t}^{\varepsilon}(y),
\qquad
\Delta_y=\sum_{m=1}^d\frac{\partial^2}{\partial y_m^2}.
\]
The expression \(\partial^2/\partial y_m^2\) is a differentiation
operator: applied to a function, it gives its second derivative in coordinate
\(y_m\). The Laplacian \(\Delta_y\) sums these derivatives across coordinates and
measures the local spatial curvature of the density
with respect to the observation \(y\); the equation states that increasing
diffusion time smooths the density according to that curvature
\cite{vincent2011denoising,ho2020ddpm,karras2022elucidating}. Because both
empirical models follow this same evolution, their fixed-time Fisher
divergences can later be integrated into a KL divergence. Appendix~
\ref{app:proof-heat-kernel} verifies the heat equation directly.
We next compare the two Hyv\"arinen scores at each smoothing scale
to obtain evidence of a change; integrating that evidence across scales
will give the KL-based mean increments.

\subsection{Closed-Form Score from Samples}
\label{sec:closed-form-score-evidence}

Section~\ref{sec:score-evidence} provides two positive, smooth reference
densities at every diffusion time \(t\):
\(\widehat p_{\infty,t}^{\varepsilon}\) for the pre-change regime and
\(\widehat p_{1,t}^{\varepsilon}\) for the post-change regime.  We now use the difference between their Hyv\"arinen scores as evidence for sequential change detection. At each fixed \(t\), the evidence is computable from the sample sets and has a negative pre-change mean and a positive post-change mean when the two smoothed empirical distributions differ.

For a positive twice differentiable density \(p\), let
\(s_p(y)=\nabla\log p(y)\) denote its score and define the Hyv\"arinen score
\[
S_H(y;p)
=\frac12\lVert s_p(y)\rVert^2+\Delta\log p(y).
\]
Both terms depend only on derivatives of \(\log p\). Consequently,
\(S_H(y;p)\) is unchanged if \(p\) is multiplied by an unknown positive
normalizing constant, which makes it suitable for comparing unnormalized
models \cite{hyvarinen2005score,wu2023scusum}.

The sign of the resulting comparison follows from the Hyv\"arinen
score-matching identity. For positive normalized densities \(p\) and \(q\),
let \(P\) and \(Q\) denote their corresponding probability distributions.
The Fisher divergence from \(P\) to \(Q\) is
\begin{equation}
D_F(P\|Q)
=\mathbb E_{Y\sim P}\!\left[\lVert s_p(Y)-s_q(Y)\rVert^2\right],
\label{eq:fisher-divergence}
\end{equation}
where \(\mathbb E_{Y\sim P}\) denotes expectation when \(Y\) is drawn from
\(P\). Under the usual integration-by-parts conditions, the Hyv\"arinen
score-matching identity states that
\begin{equation}
\mathbb E_{Y\sim P}\!\left[S_H(Y;q)-S_H(Y;p)\right]
=\frac12D_F(P\|Q)\geq0.
\label{eq:hyvarinen-score-identity}
\end{equation}
Appendix~\ref{app:proof-hyvarinen-score-identity} gives a short derivation.

We therefore define the fixed-time evidence as
\begin{equation}
\Delta_t(Y)
=
S_H(Y;\widehat p_{\infty,t}^{\varepsilon})
-
S_H(Y;\widehat p_{1,t}^{\varepsilon}).
\label{eq:fixed-time-evidence}
\end{equation}
The order of subtraction is important. If
\(Y\sim\widehat P_{\infty,t}^{\varepsilon}\), then
\eqref{eq:hyvarinen-score-identity} gives
\[
\mathbb E_{\infty}[\Delta_t(Y)]
=-\frac12D_F\!\left(
\widehat P_{\infty,t}^{\varepsilon}
\middle\|
\widehat P_{1,t}^{\varepsilon}
\right)\leq0.
\]
If \(Y\sim\widehat P_{1,t}^{\varepsilon}\), exchanging the two models gives
\[
\mathbb E_{1}[\Delta_t(Y)]
=\frac12D_F\!\left(
\widehat P_{1,t}^{\varepsilon}
\middle\|
\widehat P_{\infty,t}^{\varepsilon}
\right)\geq0.
\]
Thus \(\Delta_t\) has a nonpositive mean before the change and a
nonnegative mean afterward, with strict signs when the two smoothed
distributions differ. These signs allow CUSUM to accumulate evidence
of a change.

Because the smoothed densities are finite Gaussian mixtures,
their Hyv\"arinen scores can be computed exactly from weighted moments
of the pre- and post-change samples, without training a score network. For either reference set \(\mathcal R_r\), put \(h=\varepsilon+t\) and
\[
\alpha_j(y)=
\frac{e^{-\|y-r_j\|^2/(4h)}}{\sum_\ell e^{-\|y-r_\ell\|^2/(4h)}},
\quad
k(y)=\sum_j\alpha_j(y)r_j,
\]
and \(V(y)=\sum_j\alpha_j(y)\|r_j-k(y)\|^2\). The exact formula used for
the score comparisons is
\[
S_H(y;\widehat p_{r,t}^{\varepsilon})
=\frac{\|k(y)-y\|^2+2V(y)}{8h^2}-\frac{d}{2h}.
\]
Appendix~\ref{app:proof-closed-form-multiscale-evidence} derives these
formulas. We next combine the resulting score differences across diffusion
times to construct the sequential stopping rule.
\subsection{The \methodname{} Stopping Rule}
\label{sec:detector-construction}

To combine the evidence \(\Delta_t\) across diffusion times without
evaluating every scale, DI-SCUSUM samples one time for each observation.
Importance weighting then gives an unbiased estimate of the time integral
of the noise-averaged score evidence.
Choose a density \(f\) on \([0,\infty)\) with \(f(t)>0\)
for all \(t\geq0\), so every positive-time interval has nonzero sampling probability. For each observation \(X_n\), draw \(t_n\sim f\) and \(W_n\sim\mathcal N(0,I_d)\) independently of each other, the data stream, and all previous auxiliary draws, and define
\begingroup
\setlength{\abovedisplayskip}{4pt}
\setlength{\belowdisplayskip}{4pt}
\begin{equation}\label{eq:online-perturbation}
Y_n=X_n+\sqrt{2(\varepsilon+t_n)}\,W_n.
\end{equation}
\endgroup
The variance in this perturbation is chosen to match the same effective scale
\(h_{t_n}=\varepsilon+t_n\) in
\eqref{eq:diffused-empirical-distributions}.
For the empirical plug-in analysis, we set
\(P_\infty=\widehat P^0\) and \(P_1=\widehat Q^0\).
Conditional on \(t_n=t\), \eqref{eq:online-perturbation} therefore gives
\(Y_n\sim\widehat P_{\infty,t}^{\varepsilon}\) before the change and
\(Y_n\sim\widehat P_{1,t}^{\varepsilon}\) afterward. The fixed-time drift
identities then apply to \((t_n,Y_n)\).

We now turn the single sampled diffusion time into an estimate of the full
time integral. For a multiplier \(\lambda>0\), define
\begingroup
\setlength{\abovedisplayskip}{4pt}
\setlength{\belowdisplayskip}{4pt}
\begin{equation}
z_n^{\methodtag}
=\frac{2\lambda}{f(t_n)}\Delta_{t_n}(Y_n),
\label{eq:cf-increment}
\end{equation}
where
\[
\Delta_{t_n}(Y_n)
=S_H(Y_n;\widehat p_{\infty,t_n}^{\,\varepsilon})
-S_H(Y_n;\widehat p_{1,t_n}^{\,\varepsilon}).
\]
\endgroup
The factor \(1/f(t_n)\) gives more weight to diffusion times that
are sampled less often, so the average increment accounts for the full
time integral rather than favoring frequently sampled scales. Assume \(\int_0^\infty\mathbb E_r^{X,W}|G_t(X,W)|\,dt<\infty\)
in both regimes, with \(G_t\) defined in Section~\ref{sec:implementation}. 
Changing \(f\) preserves the mean but affects increment variability and
false-alarm calibration. The multiplier \(\lambda>0\)
sets the overall size of the increments.

DI-SCUSUM initializes \(Z_0=0\) and updates the reflected CUSUM
statistic as
\begingroup
\setlength{\abovedisplayskip}{4pt}
\setlength{\belowdisplayskip}{4pt}
\begin{equation}
Z_n=\max\{0,Z_{n-1}+z_n^{\methodtag}\}.
\label{eq:cusum-update}
\end{equation}
\endgroup The procedure raises an alarm at
\begingroup
\setlength{\abovedisplayskip}{4pt}
\setlength{\belowdisplayskip}{4pt}
\begin{equation}
T_{\methodtag}
=\inf\{n\geq1:Z_n\geq\tau\},
\label{eq:cf-recursion}
\end{equation}
\endgroup
where \(\tau>0\) is the alarm threshold. The maximum resets the statistic to zero whenever the accumulated evidence becomes negative; positive increments can raise the statistic to \(\tau\). Larger thresholds reduce false alarms but increase detection delay.

Algorithm~\ref{alg:cf-discusum} summarizes the online updates.
Appendix~\ref{app:proof-closed-form-multiscale-evidence} gives the
closed-form scores.

\begin{algorithm}[t]
\caption{\methodname{} stopping rule}
\label{alg:cf-discusum}
\begin{algorithmic}[1]
\Require Reference sets \(\mathcal U,\mathcal V\); base scale
\(\varepsilon>0\); time density \(f\); multiplier \(\lambda>0\);
threshold \(\tau>0\)
\State Initialize \(Z_0\gets0\)
\For{\(n=1,2,\ldots\)}
  \State Observe \(X_n\); sample \(t_n\sim f\) and
  \(W_n\sim\mathcal N(0,I_d)\)
  \State Set \(Y_n\gets X_n+\sqrt{2(\varepsilon+t_n)}W_n\)
  \State Evaluate the score difference \(\Delta_{t_n}(Y_n)\)
  \State \(z_n^{\methodtag}\gets
    \dfrac{2\lambda}{f(t_n)}\Delta_{t_n}(Y_n)\)
    \hfill\text{as in }\eqref{eq:cf-increment}
  \State \(Z_n\gets\max\{0,Z_{n-1}+z_n^{\methodtag}\}\)
    \hfill\text{as in }\eqref{eq:cusum-update}
  \If{\(Z_n\geq\tau\)} \State \Return \(T_{\methodtag}=n\) \EndIf
\EndFor
\end{algorithmic}
\end{algorithm}

\subsection{Sequential Guarantees}
\label{sec:sequential-guarantees}

The relative de Bruijn identity below relates the integrated score evidence to KL divergence along the common Gaussian heat flow
\cite{stam1959information,valerotoranzo2018debruijn}.
\begin{lemma}[Relative de Bruijn identity along a common heat flow]
\label{lem:heat-flow-kl}
Let \(P_t,Q_t\) have positive smooth densities satisfying
\(\partial_t p_t=\Delta p_t\) and
\(\partial_t q_t=\Delta q_t\). Under conditions that justify
differentiation and integration by parts,
\begin{equation}
\frac{d}{dt}\operatorname{KL}(P_t\|Q_t)
=-D_F(P_t\|Q_t).
\label{eq:heat-flow-kl-dissipation}
\end{equation}
Consequently, for \(0\leq a<b\),
\[
\int_a^bD_F(P_t\|Q_t)\,dt
=\operatorname{KL}(P_a\|Q_a)-\operatorname{KL}(P_b\|Q_b).
\]
If the terminal KL divergence tends to zero, the integral over
\((0,\infty)\) equals \(\operatorname{KL}(P_0\|Q_0)\).
\end{lemma}

Appendix~\ref{app:proof-heat-flow-kl-dissipation} proves the lemma.
For distinct Gaussian-smoothed empirical models, the terminal KL divergence vanishes, so
the time integral recovers the initial KL divergence. Let
\(\mathbb E_r^X\) denote expectation over an observation from regime
\(r\in\{\infty,1\}\), and let \(\mathbb E^{t,W}\) denote expectation over the
independent auxiliary draws \((t_n,W_n)\). Combining
\eqref{eq:hyvarinen-score-identity}, \eqref{eq:cf-increment}, and
\eqref{eq:heat-flow-kl-dissipation} gives
\begin{equation}
\vcenter{\hbox{\resizebox{0.86\columnwidth}{!}{$
\begin{aligned}
\mu_\infty
&\defeq \mathbb E_\infty^X\!\left[
\mathbb E^{t,W}[z_n^{\methodtag}\mid X_n]\right]
=-\lambda\operatorname{KL}
(\widehat P^\varepsilon\|\widehat Q^\varepsilon)<0,\\
\mu_1
&\defeq \mathbb E_1^X\!\left[
\mathbb E^{t,W}[z_n^{\methodtag}\mid X_n]\right]
=\lambda\operatorname{KL}
(\widehat Q^\varepsilon\|\widehat P^\varepsilon)>0.
\end{aligned}
$}}}
\label{eq:kl-drifts}
\end{equation}
Appendix~\ref{app:proof-kl-drift-identities} gives the conditioning argument.

False-alarm control requires more than a negative mean increment.
The condition \(\mathbb E_\infty[e^{z_n^{\methodtag}}]\leq1\) ensures that
the exponential of the cumulative, unreflected evidence does not increase
in conditional expectation under no change, allowing threshold-crossing
probabilities to be bounded. The next lemma gives conditions on the
increment distribution under which \(\lambda\) can enforce this inequality
or normalize the expectation to one.

\begin{lemma}[Admissible Scaling and Cram\'er Root]
\label{lem:admissible-multiplier}
\begingroup
Fix $\varepsilon$, $f$, and the reference sets. Let
$U_n=2\Delta_{t_n}(Y_n)/f(t_n)$, so $z_n^{\methodtag}=\lambda U_n$,
and write $M(s)=\mathbb E_\infty[e^{sU_n}]$.
Suppose $\mathbb E_\infty|U_n|<\infty$,
$\widehat P^\varepsilon\ne\widehat Q^\varepsilon$, and the
KL drift identity~\eqref{eq:kl-drifts} holds.
\begin{enumerate}
\item If $M(\eta)<\infty$ for some $\eta>0$, then $M(\lambda)<1$
for all sufficiently small $\lambda>0$.
\item If $1<M(b)<\infty$ for some $b>0$, there is a unique
$\lambda_\star\in(0,b)$ with $M(\lambda_\star)=1$, and
\[
M(\lambda)\leq1
\quad\Longleftrightarrow\quad
0<\lambda\leq\lambda_\star
\qquad(\lambda>0).
\]
For every $\lambda>0$, the increment $\lambda U_n$ has positive
Cram\'er root
$\theta_{\lambda,\varepsilon}=\lambda_\star/\lambda$,
with a finite moment-generating function in a neighborhood of this root.
\end{enumerate}
\endgroup
\end{lemma}
Thus $\lambda=\lambda_\star$ gives unit Cram\'er root, whereas a
smaller admissible multiplier gives a strict exponential-moment inequality.
The moment conditions concern the full importance-weighted increment;
negative drift alone does not imply them. Appendix~\ref{app:proof-admissible-multiplier}
gives the proof.

The multiplier rescales the evidence rather than changing the
observations used by the procedure. For the same observations and auxiliary
draws, multiplying every increment by \(\lambda>0\) also multiplies the
reflected statistic by \(\lambda\). Thus a threshold \(\tau\) with multiplier
\(\lambda\) is equivalent to a threshold \(\tau/\lambda\) with multiplier one:
\(T^{(\lambda)}(\tau)=T^{(1)}(\tau/\lambda)\).
If only the multiplier changes from \(\lambda\) to \(\lambda'\), rescaling the threshold from \(\tau\) to \((\lambda'/\lambda)\tau\) preserves every alarm time and requires no new threshold simulations.

For the sequential analysis, assume the increments are i.i.d. and
non-arithmetic. Before the change, assume their moment-generating function is
finite near the positive Cram\'er root \(\theta_{\lambda,\varepsilon}\)
defined by
\(
\mathbb E_\infty^{X,t,W}
[e^{\theta_{\lambda,\varepsilon}z_n^{\methodtag}}]=1
\).
After the change, assume a finite second moment, a finite mean overshoot bound
\(C_{\mathrm{os}}\), and a finite expected downward excursion,
\[
C_{\min}
=\mathbb E_1\!\left[-\inf_{m\geq0}\sum_{k=1}^m
z_k^{\methodtag}\right]<\infty.
\] 
The constant \(C_{\mathrm{os}}\) bounds, uniformly over \(\tau>0\), the expected amount by which the unreflected post-change random walk exceeds \(\tau\) at its first crossing. Both \(C_{\mathrm{os}}\) and \(C_{\min}\) are independent of \(\tau\).
These are standard renewal conditions for
CUSUM-type procedures
\cite{lorden1971procedures,moustakides1986optimal,wu2023scusum}.

Under these assumptions, the increment distribution determines
how the threshold controls false alarms and detection delay. The following
theorem gives the pre-change average run length.
\begin{theorem}[False-alarm scaling]
\label{thm:false-alarm}
Under the pre-change conditions above, as \(\tau\to\infty\),
\begin{equation}
\log\operatorname{ARL}(T_{\methodtag})
=\log\mathbb E_\infty[T_{\methodtag}]
=\theta_{\lambda,\varepsilon}\tau+O(1).
\label{eq:arl-scaling}
\end{equation}
Consequently, a target \(\operatorname{ARL}\approx\gamma\) has the
leading-order calibration
\(\tau\approx\log(\gamma)/\theta_{\lambda,\varepsilon}\).
\end{theorem}
A proof is given in Appendix~\ref{app:proof-average-run-length-scaling}.

The false-alarm theorem calibrates \(\tau\), but it does not state
how quickly a change is detected. After the change, the positive mean
\(\mu_1=\lambda\operatorname{KL}(\widehat Q^\varepsilon\|
\widehat P^\varepsilon)\) from \eqref{eq:kl-drifts} drives the statistic
toward the calibrated threshold. The next theorem bounds the resulting detection delay.
\begin{theorem}[Conditional detection delay]
\label{thm:detection-delay}
Under the post-change conditions above,
\begin{equation}
\vcenter{\hbox{\resizebox{0.86\columnwidth}{!}{$
\frac{\tau-C_{\min}}{\mu_1}
\leq
\operatorname{CADD}(T_{\methodtag})
=\mathbb E_1[T_{\methodtag}]
\leq
\frac{\tau+C_{\mathrm{os}}}{\mu_1}
$}}}
\label{eq:delay-bounds}
\end{equation}
\begin{equation}
\operatorname{CADD}(T_{\methodtag})
=
\frac{\tau}
{\lambda\operatorname{KL}
(\widehat Q^\varepsilon\|\widehat P^\varepsilon)}
+O(1).
\label{eq:first-order-delay}
\end{equation}
\end{theorem}
A proof is given in Appendix~\ref{app:proof-kl-form-first-order-detection-delay}.
At $\nu=1$, $Z_0=0$ and $T$ counts the post-change observations through the alarm; this case attains CADD for the reflected recursion.

The two theorems separate calibration from detection: the pre-change Cram\'er
root governs false alarms, while the forward KL divergence governs the
first-order delay. Combining them at a matched target
\(\operatorname{ARL}\approx\gamma\) gives
\[
\operatorname{CADD}(T_{\methodtag})
=
\frac{\log\gamma}
{\theta_{\lambda,\varepsilon}\lambda
\operatorname{KL}(\widehat Q^\varepsilon\|
\widehat P^\varepsilon)}
+O(1),
\]
which also shows why drift comparisons are meaningful only after false-alarm
calibration.

Finally, the base smoothing parameter \(\varepsilon\) creates a genuine
statistical tradeoff. It is needed to turn the empirical measures into
positive probability distributions with finite KL divergence, but additional smoothing also removes
information. Applying the same heat-flow identity with \(\varepsilon\) as the
smoothing variable gives
\[
\frac{d}{d\varepsilon}
\operatorname{KL}(\widehat Q^\varepsilon\|
\widehat P^\varepsilon)
=-D_F(\widehat Q^\varepsilon\|
\widehat P^\varepsilon)\leq0.
\]
Thus increasing \(\varepsilon\) improves regularity and numerical stability
but can reduce the KL information available for detection. The
drift identities in \eqref{eq:kl-drifts} and the guarantees in
Theorems~\ref{thm:false-alarm} and \ref{thm:detection-delay} are conditional
on the fixed reference sets and the exact empirical plug-in model. For a monitored stream whose distributions differ from the fixed
empirical models, the score increments remain computable, but their actual
means need not equal the KL expressions in \eqref{eq:kl-drifts}; the stated
guarantees therefore require the model assumptions.

The mean increment does not depend on \(f\), but its variance does.
We next choose how to estimate the diffusion integral and calibrate the
alarm threshold.

\subsection{Estimation and Calibration}
\label{sec:implementation}
The sampling density \(f\) can reduce the variability of the
increments by allocating more draws to suitable diffusion times. Define
\[
G_t(X,W)=2\Delta_t\!\left(X+\sqrt{2(\varepsilon+t)}W\right).
\]
The density that minimizes the post-change second moment is
\[
f^\star(t)\propto\sqrt{\mathbb E_1[G_t(X,W)^2]},
\]
provided its normalizing integral is finite
(Proposition~\ref{prop:optimal-time-law}; proof in
Appendix~\ref{app:proof-optimal-time-law}). In practice, the second moment
can be estimated on a pilot grid. Averaging \(M\) independent time--noise
draws for each observation preserves the mean and divides the conditional
sampling variance by \(M\), when finite, at the cost of \(M\) score
comparisons. This reduces fluctuations caused by the auxiliary draws.

Rao--Blackwellization instead removes the noise-sampling variance
by averaging over \(W\) conditional on the observation
\cite{rao1945information,blackwell1947conditional}:
\[
\overline G_t(X)=\mathbb E[G_t(X,W)\mid X].
\] A quadrature rule with times \(t_\ell\) and weights \(w_\ell\)
then approximates the remaining time integral by
\[
\widehat z_n^{\mathrm{RB}}
=\lambda\sum_{\ell=1}^Lw_\ell\overline G_{t_\ell}(X_n).
\]
Monte Carlo is unbiased under integrability; quadrature removes
time-sampling variance but introduces truncation and discretization error. 

For a target ARL \(\gamma\), the theoretical approximation
\(\tau\approx\log(\gamma)/\theta_{\lambda,\varepsilon}\) supplies an initial
threshold. We simulate no-change streams and refine the threshold by
bisection until the empirical ARL matches the target. Calibration is repeated
when \(\varepsilon\), \(f\), or the estimator changes; changing only
\(\lambda\) requires the same proportional change in \(\tau\).
Exact score evaluation costs \(O((N_0+N_1)d)\) per diffusion time. 
Appendix~\ref{app:additional-computation} gives further implementation
and approximation details.

\section{Benefits of Diffusion Integration}
\label{sec:integration-help}
Section~\ref{sec:methodology} established the validity of the integrated
detector. We now isolate why integration can retain information that a
single-scale score statistic loses.

\paragraph{Information Carried by Different Diffusion Times.}
For smooth population probability distributions \(P_\infty\) and \(P_1\), set
\(\varepsilon=0\) and evolve both under the common heat flow
\[
P_{\infty,t}=P_\infty*\mathcal N(0,2tI_d),
\qquad
P_{1,t}=P_1*\mathcal N(0,2tI_d).
\] Assume that \(\operatorname{KL}(P_1\|P_\infty)<\infty\) and that the densities satisfy the positivity and regularity conditions of Lemma~\ref{lem:heat-flow-kl}.
Define \(a(t)=D_F(P_{1,t}\|P_{\infty,t})\). Then \(a(t)/2\) is the mean
post-change score evidence at time \(t\), while
\(\int_0^\infty a(t)dt=\operatorname{KL}(P_1\|P_\infty)\) whenever the
terminal KL divergence vanishes. SCUSUM \cite{wu2023scusum} uses the score difference at \(t=0\), whereas \methodname{} estimates the integral of score differences across diffusion times.

\begin{proposition}[Information in a finite diffusion window]
\label{prop:finite-window-information}
For every \(T>0\),
\[
\int_0^T a(t)\,dt
=
\operatorname{KL}(P_1\|P_\infty)
-\operatorname{KL}(P_{1,T}\|P_{\infty,T}).
\]
If \(\operatorname{KL}(P_1\|P_\infty)>0\), the fraction recovered by times
up to \(T\) is
\[
R(T)
=
\frac{\int_0^T a(t)\,dt}
{\operatorname{KL}(P_1\|P_\infty)}
=
1-
\frac{\operatorname{KL}(P_{1,T}\|P_{\infty,T})}
{\operatorname{KL}(P_1\|P_\infty)}.
\]
\end{proposition}
Hence \(R(T)\) is exactly the fraction of forward KL information recovered by
the window \([0,T]\). Appendix~\ref{app:finite-window-information} gives the
proof and interpretation; \(f\) affects estimation variance but not \(R(T)\).

\begin{remark}[Anisotropic Gaussian comparison]
\label{rem:anisotropic-gaussian-separation}
Let \(P_\infty=\mathcal N(\theta_0,\Sigma)\) and
\(P_1=\mathcal N(\theta_1,\Sigma)\), with \(\Sigma\succ0\) and a
nonzero mean shift \(\delta=\theta_1-\theta_0\).
Define \(A=\delta^\top\Sigma^{-2}\delta\) and
\(B=\delta^\top\Sigma^{-3}\delta\). After scaling each method so that
\(\mathbb E_\infty[e^z]=1\), the post-change mean increments are
\[
I_{\mathrm{CUSUM}}=\tfrac12\delta^\top\Sigma^{-1}\delta,
\qquad
I_{\mathrm{SCUSUM}}=\frac{A^2}{2B}
\leq I_{\mathrm{CUSUM}}.
\]
Likelihood-ratio CUSUM weights covariance eigendirections by inverse
variance; SCUSUM uses inverse variance squared. A scalar multiplier cannot
align these weights when the shift spans directions with different variances,
giving a strict inequality. Equality holds when all directions carrying
the shift have the same variance.

Integration recovers the likelihood-ratio weights. For
\(D_t(y)=S_H(y;p_{\infty,t})-S_H(y;p_{1,t})\) and
\(Y_t=X+\sqrt{2t}\,W\),
\[
\mathbb E_W\!\left[2\int_0^\infty D_t(Y_t)\,dt\,\middle|\,X\right]
=\log\frac{p_1(X)}{p_\infty(X)}.
\]
Thus the ideal integrated increment equals the log-likelihood ratio,
giving \(I_{\methodtag}=I_{\mathrm{CUSUM}}\). 
Appendix~\ref{app:proof-anisotropic-gaussian-separation} gives the derivation;
Section~\ref{sec:simulation} tests this comparison.
\end{remark}

The Gaussian identity assumes exact noise averaging and integration
over all diffusion times; finite numerical implementations approximate it.
Adaptive weights across times change the target evidence and require
separate calibration (Appendix~\ref{app:weighted-discusum-details}). 

\begingroup

\section{Experiments}
\label{sec:simulation}

We first test whether diffusion integration recovers likelihood-ratio
CUSUM performance in an anisotropic Gaussian model, then compare sample-based
methods on MNIST and Oxford-IIIT Pet streams
(Figure~\ref{fig:image-experiments}). Thresholds are calibrated under no
change before evaluating detection delay.

\begin{figure*}[t]
  \centering
  \begin{minipage}[t]{0.323\textwidth}
    \centering
    \includegraphics[width=0.98\linewidth]{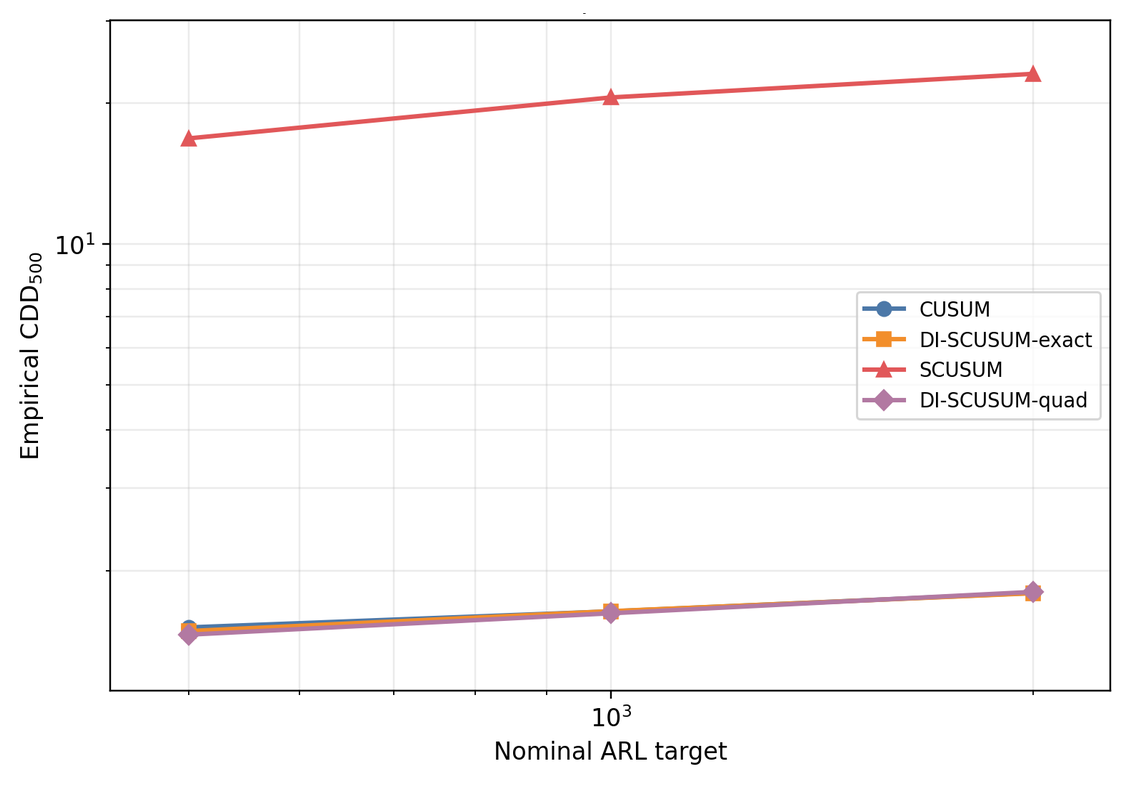}
    \vspace{-0.6em}

    {\textbf{(a)} Anisotropic Gaussian.}
  \end{minipage}\hfill
  \begin{minipage}[t]{0.323\textwidth}
    \centering
    \includegraphics[width=0.98\linewidth]{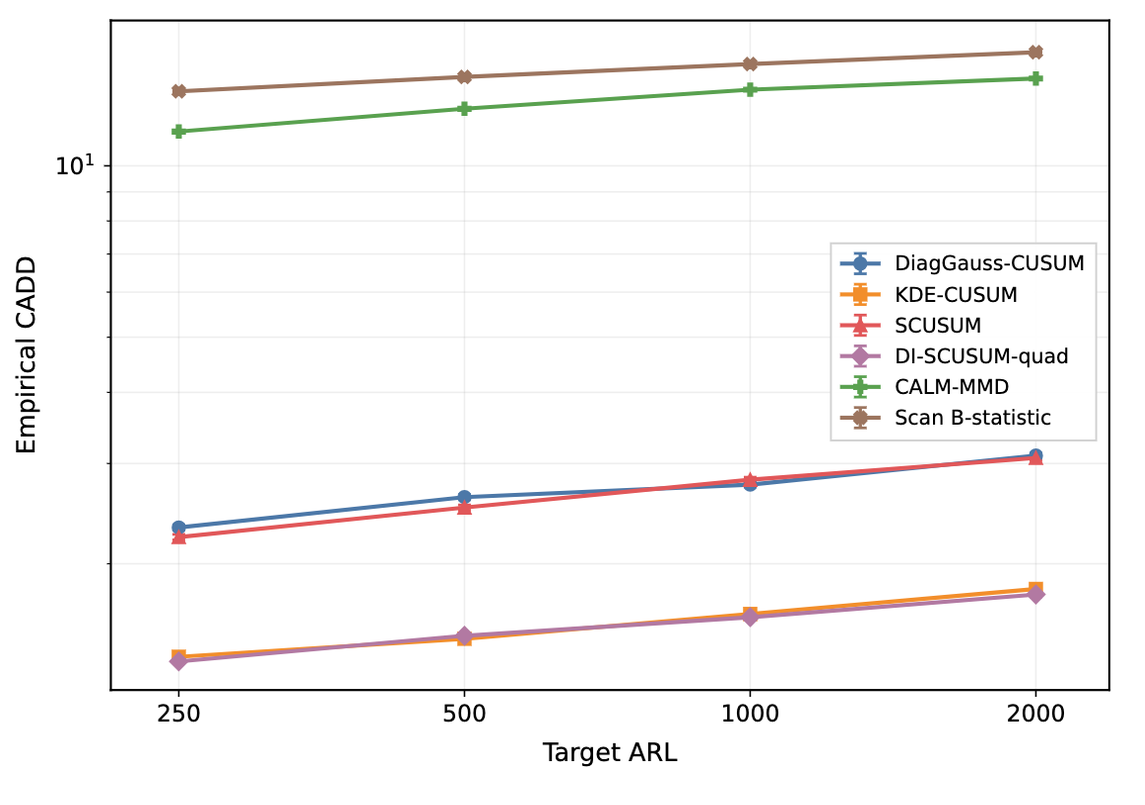}
    \vspace{-0.6em}

    {\textbf{(b)} MNIST 3 $\rightarrow$ 5.}
  \end{minipage}\hfill
  \begin{minipage}[t]{0.323\textwidth}
    \centering
    \includegraphics[width=0.98\linewidth]{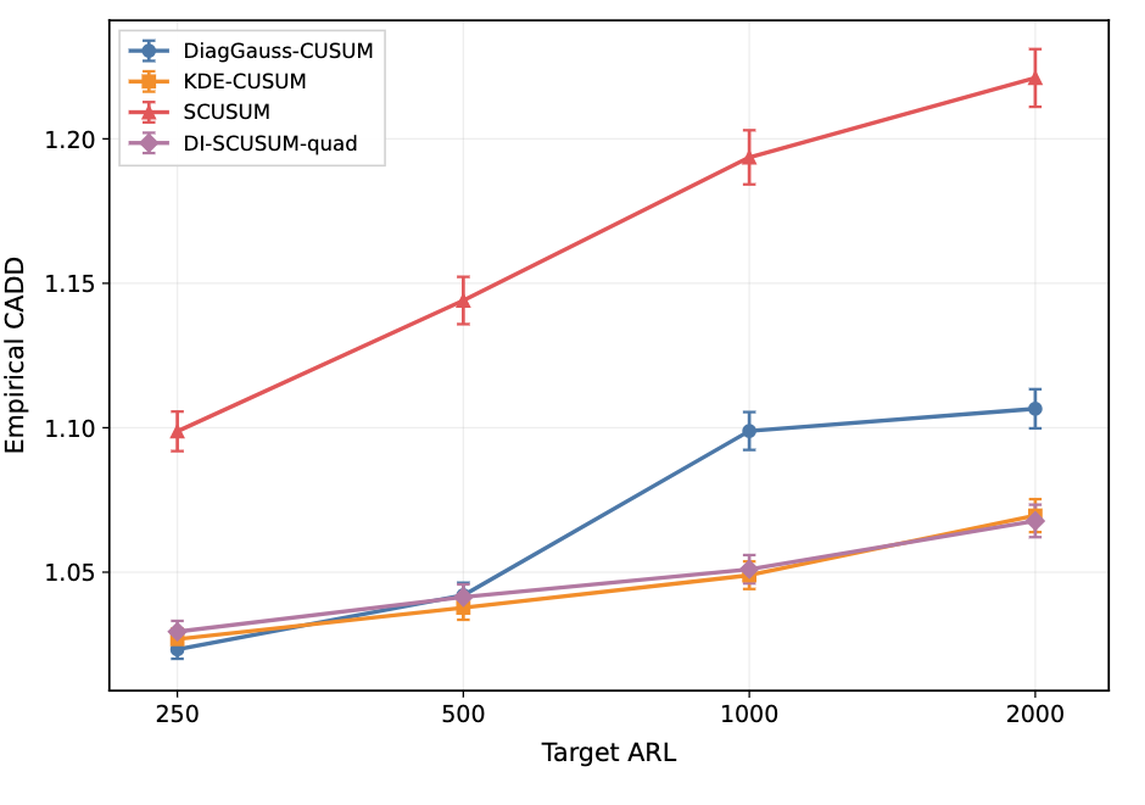}
    \vspace{-0.6em}

    \textbf{(c)} Oxford-IIIT Pet.
  \end{minipage}
  \caption{Detection delay versus nominal ARL target.
  (a) Gaussian $\mathrm{CDD}_{500,H}$; (b,c) image delay with conditional
  95\% Monte Carlo intervals. Attained ARLs are audited in the text.
  Delays count observations through the alarm; kernel methods start from
  initialized pre-change windows.}
  \label{fig:image-experiments}
\end{figure*}

\subsection{Anisotropic Gaussian Diagnostic}

\paragraph{Setup and Methods.}
The pre- and post-change means are
$\theta_0=(0,0)^\top$ and $\theta_1=(0.1,10)^\top$, with common covariance
$\Sigma=\operatorname{diag}(0.1,10)$.
After normalizing both methods so that \(\mathbb E_\infty[e^z]=1\), SCUSUM's post-change mean increment is \(0.19802\), or \(3.92\%\) of likelihood-ratio CUSUM's \(5.05\) (Remark~\ref{rem:anisotropic-gaussian-separation}).
We compare likelihood-ratio CUSUM, SCUSUM, Rao--Blackwellized
DI-SCUSUM-quad, and DI-SCUSUM-exact, whose increment equals the Gaussian
log-likelihood ratio. With $\varepsilon=0$, quadrature uses 128 log-spaced
times in $[10^{-5},10^4]$. Its coefficients $(5.02128,0.0501986)$ approximate
the exact $(5,0.05)$ with relative errors $0.426\%$ and $0.397\%$,
providing a numerical check independent of the delay estimates.
Appendix~\ref{app:additional-computation} gives the weights and thresholds.

\paragraph{Calibration and Delay.}
Each threshold is calibrated separately at
$\gamma\in\{500,1000,2000\}$ using 2,048 no-change streams,
bracket expansion, and 14 bisection steps. With horizon $H=6000$,
the calibration quantity is
$\operatorname{ARL}_H=\mathbb E_\infty[\min\{T,H\}]$, a lower bound
on uncensored ARL. Paths without an alarm are recorded at $H$.
Using 4,096 new streams with change time $\nu=500$, we estimate
\[
\operatorname{CDD}_{500,H}(T)
=\mathbb E_{500}[\min\{T,H\}-499\mid T\geq500].
\]
Conditioning excludes alarms before the change; subtracting
$\nu-1=499$ counts post-change observations through the alarm.
This horizon-censored fixed-change delay approximates its uncensored
counterpart when post-change censoring is negligible and does not estimate
Pollak's supremum over change times. 
\paragraph{Results.}
Censored no-change means are within $4.43\%$ of their targets.
Both integrated variants closely track CUSUM in
Figure~\ref{fig:image-experiments}(a), with maximum absolute delay
differences of $0.054$ observations for quadrature and $0.028$ for exact
integration. DI-SCUSUM-quad reduces fixed-change delay relative to SCUSUM
by about $91\%$, consistent with the predicted information recovery
under comparable finite-horizon calibration.

\subsection{Sample-Based Image Experiments}

\paragraph{Setup and Methods.}
MNIST~\cite{lecun1998gradient} uses digit 3 $\rightarrow$ 5,
12-dimensional PCA fitted on training pixels, and 600 reference images per
class. Oxford-IIIT Pet~\cite{parkhi2012cats} uses cats $\rightarrow$ dogs,
frozen ImageNet-pretrained ResNet-18 features~\cite{deng2009imagenet,he2016deep},
L2 normalization, and PCA fitted on \texttt{trainval}, with 450 reference
images per class. Methods share the 12-dimensional features and source
reference partition within each study.

Both studies compare diagonal-Gaussian CUSUM, Gaussian-KDE CUSUM,
SCUSUM, and DI-SCUSUM-quad with 20 time nodes and antithetic noise.
MNIST also includes CALM-MMD and the Scan $B$-statistic, using
25-observation windows and reference median-distance RBF bandwidths.
CALM uses an ARL-corrected time-varying schedule; Scan $B$ uses a
calibrated constant threshold.

\paragraph{Calibration and Delay.}
Targets are $\gamma\in\{250,500,1000,2000\}$; reference models and
drift corrections use training data only.  Calibration, ARL reporting,
and delay evaluation each use 8,192 independent paths from fixed empirical
populations. For CUSUM-type methods, calibration and reporting share a
pre-change increment pool. Delay is measured from $\nu=1$ and counts
observations through the alarm, attaining CADD under i.i.d.\ increments. 
Kernel methods instead start from initialized pre-change windows.
The 95\% Monte Carlo intervals condition on the datasets, references,
and simulation pools.

\paragraph{Results.}
On MNIST, the $5\%$ target-error and pairwise-ARL audit,
with censoring below $1\%$, passes for all six methods at targets
250 and 1000. At 500 and 2000, maximum target errors are $5.50\%$ and
$7.65\%$, and pairwise spreads are $6.25\%$ and $8.88\%$;
these points do not support strict six-method comparisons.
DI-SCUSUM-quad and SCUSUM pass their pairwise audit at all four targets:
ARL gaps are at most $3.76\%$, and CADD reductions are $39.5$--$42.7\%$.
DI-SCUSUM-quad remains close to KDE-CUSUM; kernel delays are substantially
larger at the two fully audited targets.

On Oxford-IIIT Pet, all methods pass the 99\%
lower-confidence-bound audit for $\operatorname{ARL}\geq\gamma$.
DI-SCUSUM-quad reduces CADD relative to SCUSUM by $6.3$--$12.6\%$,
with ARL gaps below $4.3\%$, and remains close to KDE-CUSUM. 
At $\gamma=2000$, diagonal-Gaussian CUSUM attains empirical ARL 3709;
Figure~\ref{fig:image-experiments}(c) plots nominal targets;
comparisons must account for the attained ARLs.

\paragraph{Comparison Scope.}
Likelihood and score methods use both pre- and post-change samples,
whereas kernel methods use only pre-change samples and update
25-observation windows. Their delay differences therefore reflect the
available information, window initialization, and detection statistic.
The MNIST kernel comparison does not isolate diffusion integration as the
sole source of improvement. The DI-SCUSUM-quad--SCUSUM comparison tests
integration more directly: both share the representation, reference models,
and reflected recursion, with thresholds calibrated separately on
no-change paths.

\endgroup

\clearpage

\subsection*{AI use statement}
In this work, we used ChatGPT and Codex (OpenAI) as supporting
tools during manuscript revision. Editorial assistance included improving
grammar, wording, clarity, flow, and presentation of author-provided text,
formatting figures and references, and checking literature and citations.
The tools also supported the review and refinement of mathematical
derivations and proof explanations, diagnostic numerical checks of
identities, and interpretation of existing experimental results. The
diagnostic checks did not generate the reported detection experiments.
The authors retain responsibility for evaluating AI-assisted suggestions
and for the final content of the paper, including its text, mathematical
claims, proofs, code, citations, and experimental results.

\makeatletter
\renewcommand{\@biblabel}[1]{}
\makeatother


\clearpage
\appendix
\thispagestyle{empty}


\onecolumn
\aistatstitle{Supplementary Materials}

\setcounter{section}{0}
\renewcommand{\thesection}{A}
\renewcommand{\thesubsection}{A.\arabic{subsection}}

\section{Proofs}

\subsection{Parameter Roles}
\label{app:parameter-roles}
The parameters \(\varepsilon\) and \(t\) therefore serve two
different purposes. The base scale \(\varepsilon\) is applied once to replace
the two discrete empirical measures by smooth distributions with common
support. Starting from these fixed reference models, varying \(t\) produces
increasingly smoothed views of the same two distributions. At small \(t\), the
score contrast remains sensitive to local differences near the reference
samples. As \(t\) increases, the added noise variance \(2t\) increases and
both distributions become smoother. Because the same Gaussian smoothing is
applied to both distributions, their KL divergence decreases and tends to
zero as \(t\to\infty\); it does not diverge. Large \(t\) therefore contributes
coarser but progressively weaker evidence. DI-SCUSUM integrates the score
contrast over the full path, combining the informative contributions across
scales rather than selecting one smoothing level in advance.

The four tuning quantities should be kept conceptually separate. The base
scale \(\varepsilon\) changes the two empirical plug-in
probability distributions themselves by
setting the minimum amount of Gaussian smoothing. The density \(f\) does not
change those probability distributions; it determines which diffusion times are sampled more often
and therefore affects the variance of the integral estimator. The multiplier
\(\lambda\) rescales every increment, while \(\tau\) determines how much
accumulated evidence is required before stopping. In particular, changing
\(f\) is an estimation choice, whereas changing \(\varepsilon\) changes the
statistical problem being detected.

Small \(\varepsilon\) keeps the smoothed probability distributions close to their empirical
measures but yields sharper scores and potentially large KL values when
supports differ; large \(\varepsilon\) improves numerical stability while
attenuating distributional differences. The time density \(f\) changes the
variance, not the target integral, provided it has adequate support.

\subsection{\ch{Gaussian-Mixture Heat-Flow Calculation}}
\label{app:proof-heat-kernel}
\begin{proof}

By convolution with the empirical measures,
\begin{align*}
\widehat p_{\infty,t}^{\varepsilon}(y)
&=\frac{1}{N_0}\sum_{i=1}^{N_0}\varphi_{h_t}(y-u_i),\\
\widehat p_{1,t}^{\varepsilon}(y)
&=\frac{1}{N_1}\sum_{j=1}^{N_1}\varphi_{h_t}(y-v_j),
\end{align*}
where, for \(s>0\),
\[
\varphi_s(x)=(4\pi s)^{-d/2}
\exp\!\left(-\frac{\|x\|^2}{4s}\right).
\]
Direct differentiation gives
\[
\partial_s\varphi_s(x)
=
\left(-\frac{d}{2s}+\frac{\|x\|^2}{4s^2}\right)\varphi_s(x).
\]
Moreover,
\[
\nabla\varphi_s(x)=-\frac{x}{2s}\varphi_s(x),
\qquad
\Delta\varphi_s(x)
=
\left(-\frac{d}{2s}+\frac{\|x\|^2}{4s^2}\right)\varphi_s(x).
\]
Hence \(\partial_s\varphi_s=\Delta\varphi_s\). Since
\(h_t=\varepsilon+t\), differentiation with respect to \(t\) is the same as
differentiation with respect to \(h_t\); finite sums commute with both
operators. Therefore both smoothed empirical densities satisfy the heat
equation \(\partial_t\widehat p_{r,t}^{\varepsilon}
=\Delta\widehat p_{r,t}^{\varepsilon}\) for
\(r\in\{\infty,1\}\).

\end{proof}

\subsection{\ch{Hyv\"arinen Score-Matching Identity}}
\label{app:proof-hyvarinen-score-identity}
\begin{proof}
Write $s_p=\nabla\log p$ and $s_q=\nabla\log q$. By definition,
\begin{align*}
\mathbb{E}_{P}\bigl[S_H(Y;q)-S_H(Y;p)\bigr]
&=
\int_{\mathbb{R}^d} p(y)
\left\{
\frac{1}{2}\lVert s_q(y)\rVert^2
-\frac{1}{2}\lVert s_p(y)\rVert^2
\right\}\,dy \\
&\quad+
\int_{\mathbb{R}^d} p(y)
\left\{
\Delta\log q(y)-\Delta\log p(y)
\right\}\,dy .
\end{align*}
Integration by parts gives
\[
\int_{\mathbb{R}^d}p\,\Delta\log q\,dy
=-\int_{\mathbb{R}^d}
\langle\nabla p,\nabla\log q\rangle\,dy
=-\int_{\mathbb{R}^d}p\,\langle s_p,s_q\rangle\,dy,
\]
and, similarly,
\[
\int_{\mathbb{R}^d}p\,\Delta\log p\,dy
=-\int_{\mathbb{R}^d}p\,\lVert s_p\rVert^2\,dy.
\]
Substitution yields
\begin{align*}
\mathbb{E}_{P}\bigl[S_H(Y;q)-S_H(Y;p)\bigr]
&=\frac{1}{2}\mathbb{E}_{P}\lVert s_q(Y)\rVert^2
  -\frac{1}{2}\mathbb{E}_{P}\lVert s_p(Y)\rVert^2 \\
&\quad-\mathbb{E}_{P}\langle s_p(Y),s_q(Y)\rangle
  +\mathbb{E}_{P}\lVert s_p(Y)\rVert^2 \\
&=\frac{1}{2}\mathbb{E}_{P}
  \left\lVert s_p(Y)-s_q(Y)\right\rVert^2 \\
&=\frac{1}{2}D_F(P\mathbin{\|}Q).
\end{align*}
The equivalent identity follows by multiplying both sides by $-1$.
\end{proof}

\subsection{Closed-Form Gaussian-Mixture Score Formulas}
\label{app:proof-closed-form-multiscale-evidence}

For the pre-change reference set \(\mathcal U\), define
\begin{equation}
\begin{aligned}
\alpha_i^{\infty,\varepsilon}(y;t)
&=
\frac{\exp(-\|y-u_i\|^2/(4h_t))}
{\sum_{\ell=1}^{N_0}\exp(-\|y-u_\ell\|^2/(4h_t))},
\\
k_t^{\infty,\varepsilon}(y)
&=\sum_{i=1}^{N_0}\alpha_i^{\infty,\varepsilon}(y;t)u_i,
\\
V_t^{\infty,\varepsilon}(y)
&=\sum_{i=1}^{N_0}\alpha_i^{\infty,\varepsilon}(y;t)
\|u_i-k_t^{\infty,\varepsilon}(y)\|^2.
\end{aligned}
\label{eq:barycentric-quantities}
\end{equation}
The post-change quantities
\(\alpha_j^{1,\varepsilon}\), \(k_t^{1,\varepsilon}\), and
\(V_t^{1,\varepsilon}\) are defined analogously using \(\mathcal V\).

\begin{proposition}[Gaussian-mixture score identities]
\label{prop:closed-form-empirical-scores}
For \(r\in\{\infty,1\}\), \(t\geq0\), and \(y\in\mathbb R^d\),
\begin{align}
\nabla\log\widehat p_{r,t}^{\varepsilon}(y)
&=\frac{k_t^{r,\varepsilon}(y)-y}{2h_t},
\label{eq:closed-form-score}\\
\operatorname{div}k_t^{r,\varepsilon}(y)
&=\frac{V_t^{r,\varepsilon}(y)}{2h_t},
\label{eq:barycentric-divergence}\\
S_H(y;\widehat p_{r,t}^{\varepsilon})
&=
\frac{\|k_t^{r,\varepsilon}(y)-y\|^2
+2V_t^{r,\varepsilon}(y)}
{8h_t^2}
-\frac{d}{2h_t}.
\label{eq:closed-form-hyvarinen}
\end{align}
Consequently, both terms in \(\Delta_t(y)\) can be evaluated exactly from
the corresponding reference set in \(O(N_r d)\) operations.
\end{proposition}

\begin{proof}
\medskip\noindent\textit{Diffusion score.}
For the pre-change density,
\[
\widehat p_{\infty,t}^{\varepsilon}(y)
=\frac{1}{N_0}\sum_{i=1}^{N_0}
(4\pi(\varepsilon+t))^{-d/2}
\exp\!\left(-\frac{\lVert y-u_i\rVert^2}{4(\varepsilon+t)}\right).
\]
Differentiating gives
\[
\nabla\widehat p_{\infty,t}^{\varepsilon}(y)
=\frac{1}{N_0}\sum_{i=1}^{N_0}
\varphi_{\varepsilon+t}(y-u_i)
\left(-\frac{y-u_i}{2(\varepsilon+t)}\right).
\]
Dividing by $\widehat p_{\infty,t}^{\varepsilon}(y)$ yields
\begin{align*}
\nabla\log\widehat p_{\infty,t}^{\varepsilon}(y)
&=\sum_{i=1}^{N_0}\alpha_i^{\infty,\varepsilon}(y;t)
  \frac{u_i-y}{2(\varepsilon+t)} \\
&=\frac{k_t^{\infty,\varepsilon}(y)-y}{2(\varepsilon+t)}.
\end{align*}
The proof for $\widehat p_{1,t}^{\varepsilon}$ is identical.

\medskip\noindent\textit{Barycentric divergence.}
We prove the pre-change identity; the post-change identity is identical. Let
$h=\varepsilon+t$ and $k(y)=k_t^{\infty,\varepsilon}(y)$. For notational
convenience, write $\alpha_i(y)=\alpha_i^{\infty,\varepsilon}(y;t)$. Then
\[
\alpha_i(y)
=\frac{\exp\!\left(-\lVert y-u_i\rVert^2/(4h)\right)}
{\displaystyle\sum_{\ell=1}^{N_0}
 \exp\!\left(-\lVert y-u_\ell\rVert^2/(4h)\right)}.
\]
For $r\in\{1,\ldots,d\}$,
\begin{align*}
\frac{\partial}{\partial y_r}\log\alpha_i(y)
&=-\frac{y_r-u_{i,r}}{2h}
  +\sum_{\ell=1}^{N_0}\alpha_\ell(y)
   \frac{y_r-u_{\ell,r}}{2h} \\
&=\frac{u_{i,r}-k_r(y)}{2h}.
\end{align*}
Thus
\[
\frac{\partial}{\partial y_r}\alpha_i(y)
=\alpha_i(y)\frac{u_{i,r}-k_r(y)}{2h}.
\]
Since $k_r(y)=\sum_{i=1}^{N_0}\alpha_i(y)u_{i,r}$,
\begin{align*}
\frac{\partial}{\partial y_r}k_r(y)
&=\sum_{i=1}^{N_0}u_{i,r}
  \frac{\partial}{\partial y_r}\alpha_i(y) \\
&=\frac{1}{2h}\sum_{i=1}^{N_0}
  \alpha_i(y)u_{i,r}\bigl(u_{i,r}-k_r(y)\bigr).
\end{align*}
Because
$\sum_{i=1}^{N_0}\alpha_i(y)(u_{i,r}-k_r(y))=0$, we may subtract
$k_r(y)$ from the first factor $u_{i,r}$ to obtain
\[
\frac{\partial}{\partial y_r}k_r(y)
=\frac{1}{2h}\sum_{i=1}^{N_0}
\alpha_i(y)\bigl(u_{i,r}-k_r(y)\bigr)^2.
\]
Summing over $r=1,\ldots,d$ gives the desired divergence formula.

\medskip\noindent\textit{Hyv\"arinen score.}
For the pre-change probability distribution,
\[
s_t^{\infty,\varepsilon}(y)
\coloneqq
\nabla\log\widehat p_{\infty,t}^{\varepsilon}(y)
=\frac{k_t^{\infty,\varepsilon}(y)-y}{2(\varepsilon+t)}.
\]
Therefore,
\[
\frac{1}{2}\left\lVert s_t^{\infty,\varepsilon}(y)\right\rVert^2
=\frac{\left\lVert k_t^{\infty,\varepsilon}(y)-y\right\rVert^2}
       {8(\varepsilon+t)^2}.
\]
Moreover,
\begin{align*}
\operatorname{div}s_t^{\infty,\varepsilon}(y)
&=\frac{\operatorname{div}k_t^{\infty,\varepsilon}(y)}
        {2(\varepsilon+t)}
  -\frac{\operatorname{div}y}{2(\varepsilon+t)} \\
&=\frac{\operatorname{div}k_t^{\infty,\varepsilon}(y)}
        {2(\varepsilon+t)}
  -\frac{d}{2(\varepsilon+t)}.
\end{align*}
Since
\[
S_H(y;r)
=\frac{1}{2}\lVert\nabla\log r(y)\rVert^2
+\operatorname{div}\!\bigl(\nabla\log r(y)\bigr),
\]
the first formula follows. The proof for the post-change
probability distribution is identical.
\end{proof}

\subsection{\ch{Proof of Lemma~\ref{lem:heat-flow-kl}}}
\label{app:proof-heat-flow-kl-dissipation}

\begin{proof}
Differentiating the relative entropy gives
\begin{align*}
\frac{d}{dt}\operatorname{KL}(P_t\mathbin{\|}Q_t)
&=
\frac{d}{dt}\int_{\mathbb{R}^d}
p_t\log\frac{p_t}{q_t}\,dy \\
&=
\int_{\mathbb{R}^d}(\partial_t p_t)
\log\frac{p_t}{q_t}\,dy \\
&\quad+
\int_{\mathbb{R}^d}p_t
\left(
\frac{\partial_t p_t}{p_t}
-\frac{\partial_t q_t}{q_t}
\right)\,dy .
\end{align*}
Since $\int_{\mathbb{R}^d}\partial_t p_t\,dy=0$, the heat equations imply
\[
\frac{d}{dt}\operatorname{KL}(P_t\mathbin{\|}Q_t)
=\int_{\mathbb{R}^d}(\Delta p_t)
  \log\frac{p_t}{q_t}\,dy
-\int_{\mathbb{R}^d}p_t\frac{\Delta q_t}{q_t}\,dy.
\]
Integration by parts in the first term gives
\[
\int_{\mathbb{R}^d}(\Delta p_t)
\log\frac{p_t}{q_t}\,dy
=-\int_{\mathbb{R}^d}p_t
\left\langle
  \nabla\log p_t,
  \nabla\log\frac{p_t}{q_t}
\right\rangle dy.
\]
For the second term, use
\[
\frac{\Delta q_t}{q_t}
=\Delta\log q_t+\lVert\nabla\log q_t\rVert^2
\]
and
\[
\int_{\mathbb{R}^d}p_t\Delta\log q_t\,dy
=-\int_{\mathbb{R}^d}p_t
\left\langle\nabla\log p_t,\nabla\log q_t\right\rangle dy.
\]
Combining these identities yields
\begin{align*}
\frac{d}{dt}\operatorname{KL}(P_t\mathbin{\|}Q_t)
&=-\int_{\mathbb{R}^d}p_t
\left\lVert\nabla\log p_t-\nabla\log q_t\right\rVert^2dy \\
&=-D_F(P_t\mathbin{\|}Q_t).
\end{align*}
Integrating from $0$ to $R$ gives
\[
\int_0^R D_F(P_t\mathbin{\|}Q_t)\,dt
=\operatorname{KL}(P_0\mathbin{\|}Q_0)
-\operatorname{KL}(P_R\mathbin{\|}Q_R).
\]
Letting $R\to\infty$ and using
$\operatorname{KL}(P_R\mathbin{\|}Q_R)\to0$ proves the result.
\end{proof}

\subsection{Proof of the KL Drift Identities}
\label{app:proof-kl-drift-identities}

\begin{proof}
Condition on the sampled diffusion time $t$. Under the post-change empirical
model, $X\sim\widehat Q^0$, and therefore
\[
Y\mid t
=X+\sqrt{2(\varepsilon+t)}\,W
\sim\widehat P_{1,t}^{\varepsilon}.
\]
By the Hyv\"arinen-score identity,
\begin{align*}
&\mathbb{E}_{Y\sim\widehat P_{1,t}^{\varepsilon}}
\left[
S_H\!\left(Y;\widehat p_{\infty,t}^{\varepsilon}\right)
-S_H\!\left(Y;\widehat p_{1,t}^{\varepsilon}\right)
\right]
\\
&\qquad
=\frac{1}{2}D_F\!\left(
\widehat P_{1,t}^{\varepsilon}
\mathbin{\|}
\widehat P_{\infty,t}^{\varepsilon}
\right).
\end{align*}
Consequently,
\begin{align*}
\mathbb{E}_1\!\left[z_{\lambda,\varepsilon}^{\methodtag}\right]
&=\int_0^\infty
\frac{2\lambda}{f(t)}\,
\frac{1}{2}D_F\!\left(
\widehat P_{1,t}^{\varepsilon}
\mathbin{\|}
\widehat P_{\infty,t}^{\varepsilon}
\right)f(t)\,dt \\
&=\lambda\int_0^\infty D_F\!\left(
\widehat P_{1,t}^{\varepsilon}
\mathbin{\|}
\widehat P_{\infty,t}^{\varepsilon}
\right)\,dt.
\end{align*}
The heat-flow KL dissipation identity gives
\[
\int_0^\infty D_F\!\left(
\widehat P_{1,t}^{\varepsilon}
\mathbin{\|}
\widehat P_{\infty,t}^{\varepsilon}
\right)\,dt
=\operatorname{KL}\!\left(
\widehat Q^{\varepsilon}
\mathbin{\|}
\widehat P^{\varepsilon}
\right).
\]
Here the limiting KL divergence is zero because the two Gaussian mixtures
have the same increasing covariance scale and finite-support mixing measures.
Hence
\[
\mathbb{E}_1\!\left[z_{\lambda,\varepsilon}^{\methodtag}\right]
=\lambda\operatorname{KL}\!\left(
\widehat Q^{\varepsilon}\mathbin{\|}\widehat P^{\varepsilon}
\right).
\]

Under the pre-change empirical model,
$Y\mid t\sim\widehat P_{\infty,t}^{\varepsilon}$. The
Hyv\"arinen-score identity now gives
\begin{align*}
&\mathbb{E}_{Y\sim\widehat P_{\infty,t}^{\varepsilon}}
\left[
S_H\!\left(Y;\widehat p_{\infty,t}^{\varepsilon}\right)
-S_H\!\left(Y;\widehat p_{1,t}^{\varepsilon}\right)
\right]
\\
&\qquad
=-\frac{1}{2}D_F\!\left(
\widehat P_{\infty,t}^{\varepsilon}
\mathbin{\|}
\widehat P_{1,t}^{\varepsilon}
\right).
\end{align*}
Averaging over $t$ and applying the heat-flow identity yields
\begin{align*}
\mathbb{E}_{\infty}\!\left[z_{\lambda,\varepsilon}^{\methodtag}\right]
&=-\lambda\int_0^\infty D_F\!\left(
\widehat P_{\infty,t}^{\varepsilon}
\mathbin{\|}
\widehat P_{1,t}^{\varepsilon}
\right)\,dt \\
&=-\lambda\operatorname{KL}\!\left(
\widehat P^{\varepsilon}
\mathbin{\|}
\widehat Q^{\varepsilon}
\right).
\end{align*}
\end{proof}

\subsection{Proof of Lemma~\ref{lem:admissible-multiplier}}
\label{app:proof-admissible-multiplier}
\begin{proof}
\begingroup
By \eqref{eq:kl-drifts},
\[
\mathbb E_\infty U_n
=-\operatorname{KL}(\widehat P^\varepsilon\|\widehat Q^\varepsilon)<0.
\]
If $M(\eta)<\infty$, dominated convergence gives the right derivative
$M'(0+)=\mathbb E_\infty U_n<0$.
Indeed, for $0<s\leq\eta/2$, the absolute difference quotient
$|(e^{sU_n}-1)/s|$ is bounded by $|U_n|$ on $\{U_n<0\}$
and by $U_n e^{\eta U_n/2}$ on $\{U_n\geq0\}$; both are integrable.
Since $M(0)=1$, this proves part~1.

For part~2, $M(b)<\infty$ ensures continuity on $[0,b]$.
Moreover, $M$ is strictly convex on its finite domain because
$\mathbb P_\infty(U_n\ne0)>0$. Part~1 gives $M(s)<1$ for small
$s>0$, whereas $M(b)>1$. Continuity yields a root
$\lambda_\star\in(0,b)$. Strict convexity gives
$M(s)<1$ for $0<s<\lambda_\star$ and $M(s)>1$ for
$s>\lambda_\star$ wherever $M(s)$ is finite; outside that domain
$M(s)=\infty$. This proves uniqueness and the stated admissible interval.

Finally,
\[
\mathbb E_\infty[e^{\theta z_n^{\methodtag}}]=M(\theta\lambda),
\]
so the positive root is $\lambda_\star/\lambda$.
Because $\lambda_\star<b$, the moment-generating function is finite
in a neighborhood of this root. Induction in the reflected recursion
gives $Z_n^{(\lambda)}=\lambda Z_n^{(1)}$, proving the stated
stopping-time identity.
\endgroup
\end{proof}

\subsection{Proof of Theorem~\ref{thm:false-alarm}}
\label{app:proof-average-run-length-scaling}

\begin{proof}
Let
\[
S_n=\sum_{k=1}^{n}z_k,
\]
where the $z_k$ are independent copies of
$z_{\lambda,\varepsilon}^{\methodtag}$ under $P_\infty$. By the
pre-change condition of Theorem~\ref{thm:false-alarm},
\[
M_n
\coloneqq
\exp\!\left(\theta_{\lambda,\varepsilon}S_n\right)
\]
is a nonnegative martingale under $P_\infty$. The reflected CUSUM statistic
has the representation
\[
Z_n=S_n-\min_{0\leq k\leq n}S_k.
\]
Thus a crossing of $Z_n$ above $\tau$ occurs when an excursion of the
negative-drift random walk exceeds height $\tau$. The Cram\'er--Lundberg
estimate for the ascending ladder-height process gives
\[
\mathbb{P}_\infty
\bigl(\text{one excursion exceeds }\tau\bigr)
\asymp
C_{\lambda,\varepsilon}
\exp\!\left(-\theta_{\lambda,\varepsilon}\tau\right)
\]
for some constant $C_{\lambda,\varepsilon}>0$. The excursions are
regenerative at returns of the reflected process to zero. Consequently, the
expected number of excursions before a false alarm grows on the order of
$\exp(\theta_{\lambda,\varepsilon}\tau)$. The finite mean cycle length and
renewal constants contribute only a multiplicative constant, giving
\(\log\mathbb E_\infty[T_{\methodtag}]
=\theta_{\lambda,\varepsilon}\tau+O(1)\).
\end{proof}

\subsection{Proof of Theorem~\ref{thm:detection-delay}}
\label{app:proof-kl-form-first-order-detection-delay}

\begin{proof}

Conditional on no alarm before a change at time $\nu$, the state satisfies
$0\leq Z_{\nu-1}<\tau$.  Couple two post-change recursions by giving them
the same increments.  The recursion $z\mapsto\max\{0,z+x\}$ is
nondecreasing in $z$, so the time needed to reach $\tau$ is nonincreasing
in the starting state.  The conditional delay for any $\nu$ is therefore
at most the delay from $Z_0=0$.  When $\nu=1$, the procedure starts from
$Z_0=0$, and hence
\[
\operatorname{CADD}(T_{\methodtag})
=\mathbb E_1[T_{\methodtag}].
\]

After the change, let
\[
z_k=z_{\lambda,\varepsilon,k}^{\methodtag},
\qquad
S_n=\sum_{k=1}^{n}z_k.
\]
The KL drift identity gives
\[
\mu_1
=\lambda\operatorname{KL}\!\left(
\widehat Q^{\varepsilon}\mathbin{\|}\widehat P^{\varepsilon}
\right)>0.
\]
For the reflected recursion started at $Z_0=0$,
\[
Z_n=S_n-\min_{0\leq k\leq n}S_k\geq S_n.
\]
Therefore the reflected statistic crosses $\tau$ no later than the
unreflected random walk, so
\[
T_{\methodtag}\leq H_\tau.
\]
Wald's identity and the bounded mean overshoot give
\begin{align*}
\mu_1\mathbb{E}_1[H_\tau]
&=\mathbb{E}_1[S_{H_\tau}] \\
&=\tau+\mathbb{E}_1[S_{H_\tau}-\tau] \\
&\leq\tau+C_{\mathrm{os}}.
\end{align*}
Hence
\[
\mathbb{E}_1\!\left[T_{\methodtag}\right]
\leq
\frac{\tau+C_{\mathrm{os}}}{\mu_1}.
\]
For the matching lower bound, write
\[
T=T_{\methodtag}.
\]
At the stopping time,
\begin{align*}
S_T
&=Z_T+\min_{0\leq k\leq T}S_k \\
&\geq \tau+\inf_{n\geq0}S_n.
\end{align*}
Since the upper bound above implies $\mathbb{E}_1[T]<\infty$, Wald's
identity applies and yields
\[
\mu_1\mathbb{E}_1[T]
=\mathbb{E}_1[S_T]
\geq\tau-C_{\min}.
\]
Combining the two bounds and substituting the value of $\mu_1$ proves the
result.
\end{proof}

\subsection{Time-Density Optimization and Monte Carlo Estimation}
\label{app:proof-optimal-time-law}


Although \(f\) does not change the mean increment, it controls the estimator's
second moment. Let \(X\) follow the post-change plug-in probability distribution, let
\(W\sim\mathcal N(0,I_d)\) be independent of \(X\), and define

\[
Y_t=X+\sqrt{2(\varepsilon+t)}W.
\]
With
\[
G_t(X,W)
=2\{S_H(Y_t;\widehat p_{\infty,t}^{\varepsilon})
-S_H(Y_t;\widehat p_{1,t}^{\varepsilon})\},
\]
write
\[
m_2(t)=\mathbb E_1[G_t(X,W)^2].
\]
The optimal importance-sampling density is determined by the
following second-moment calculation.

\begin{proposition}[Second-moment-optimal time density]
\label{prop:optimal-time-law}
Assume
\(0<\int_0^\infty\sqrt{m_2(t)}\,dt<\infty\). Among all admissible densities
\(f\), the post-change second moment of the increment satisfies
\begin{equation}
\mathbb E_1[(z_n^{\methodtag})^2]
=\lambda^2\int_0^\infty\frac{m_2(t)}{f(t)}\,dt
\geq\lambda^2\left(\int_0^\infty\sqrt{m_2(t)}\,dt\right)^2.
\label{eq:optimal-time-law-bound}
\end{equation}
Equality holds for
\[
f^\star(t)=
\frac{\sqrt{m_2(t)}}
{\int_0^\infty\sqrt{m_2(s)}\,ds}.
\]
\end{proposition}

Thus \(f^\star\) samples times in proportion to their root second moment;
in practice, \(m_2\) can be estimated on a pilot grid. The one-sample
increment in Section~\ref{sec:detector-construction} is the least expensive
Monte Carlo estimator. With a larger per-observation budget, one may average
\(M\) independent time--noise pairs,

\[
\widehat z_{n,M}^{\mathrm{MC}}
=\frac{\lambda}{M}\sum_{m=1}^M
\frac{G_{t_{n,m}}(X_n,W_{n,m})}{f(t_{n,m})}.
\]

\begin{proof}

Cauchy--Schwarz applied to
\(\sqrt{m_2(t)/f(t)}\) and \(\sqrt{f(t)}\) gives
\[
\begin{aligned}
\left(\int_0^\infty\sqrt{m_2(t)}\,dt\right)^2
&\leq
\left(\int_0^\infty\frac{m_2(t)}{f(t)}\,dt\right)
\\[-2pt]
&\quad\times\left(\int_0^\infty f(t)\,dt\right).
\end{aligned}
\]
The last factor is one, and equality requires
\(f(t)\propto\sqrt{m_2(t)}\).

\end{proof}

\subsection{Additional Computation and Approximation Details}
\label{app:additional-computation}

\paragraph{Gaussian Diagnostic Implementation.}

For diagonal variances $\sigma_i$, the Rao--Blackwellized increment is
\[
z^{\mathrm{quad}}(X)=
\sum_{i=1}^d\bigl\{(X_i-\theta_{0,i})^2-(X_i-\theta_{1,i})^2\bigr\}q_i,
\qquad
q_i=\sum_{\ell=1}^{128}\frac{w_\ell}{(\sigma_i+2t_\ell)^2}.
\]
For the log-spaced grid in $[10^{-5},10^4]$, weights are
$w_1=t_2-t_1$, $w_{128}=t_{128}-t_{127}$, and
$w_\ell=(t_{\ell+1}-t_{\ell-1})/2$ at interior nodes.
The coefficients $(5.02128,0.0501986)$ have relative errors
$0.426\%$ and $0.397\%$ against the exact $(5,0.05)$.
Calibration uses bracket expansion and 14 bisection steps; the fixed-change
delay estimates retain 1,507--3,317 paths without a prior alarm.
The run uses seed 123, 32-bit PyTorch arithmetic, and an NVIDIA RTX PRO
6000 Blackwell GPU. Table~\ref{tab:simulation-thresholds} lists thresholds.

\paragraph{Gaussian Diagnostic Interpretation.}
The Gaussian panel reports the four methods at the same nominal ARL targets,
with horizon-censored delay values read directly from the plotted points. CUSUM and
DI-SCUSUM-exact coincide analytically; DI-SCUSUM-quad follows closely, with
residual error from quadrature. SCUSUM is slower in this example because its score difference weights
covariance directions by inverse variance squared, whereas the log-likelihood
ratio uses inverse variance. Thus the
panel checks the numerical identity and predicted information gap without a
second table.

\paragraph{Pilot Time-Density Selection and Monte Carlo.}

The optimal density \(f^\star(t)\propto\sqrt{m_2(t)}\) depends on a
post-change second moment that is usually unavailable in closed form. It can
be estimated offline by simulating from the post-change reference model over
a pilot grid of diffusion times, smoothing the resulting moment estimates,
and normalizing their square roots. A simpler alternative is to choose a
parametric family for \(f\) and minimize the estimated second moment within
that family. Averaging \(M\) independently sampled time--noise pairs retains
the same target as the one-pair estimator and reduces the conditional Monte
Carlo variance at a proportional increase in score evaluations.

\paragraph{Rao--Blackwellization and Quadrature.}

Replacing \(G_t(X,W)\) by
\(\overline G_t(X)=\mathbb E[G_t(X,W)\mid X]\) removes the auxiliary-noise
variance without changing the conditional target. A deterministic quadrature
rule can then approximate the remaining time integral. Unlike Monte Carlo,
quadrature has no time-sampling variance, but its finite nodes introduce
discretization error and, on an unbounded domain, truncation error. The finite
window identity below quantifies the information omitted by truncation; the
remaining discretization error depends on the smoothness of
\(t\mapsto\overline G_t(X)\) and on the chosen rule.

\paragraph{Threshold Calibration.}

For a target \(\operatorname{ARL}=\gamma\), the renewal approximation
\(\tau_0=\log(\gamma)/\theta_{\lambda,\varepsilon}\) provides an initial
threshold. Finite-threshold overshoot and the complete increment distribution
affect the realized ARL, so we simulate no-change streams, enlarge a threshold
bracket until it contains the target, and refine it by bisection. The procedure
must be repeated when $\varepsilon$, $f$, the number of Monte Carlo samples,
or the quadrature rule changes. If only the multiplier changes from
$\lambda$ to $\lambda'$, the threshold changes exactly from $\tau$ to
$(\lambda'/\lambda)\tau$, without further simulation. Detector comparisons are
made only after calibration to comparable false-alarm levels.

\paragraph{Finite Diffusion Window.}
If the time density is supported on \(0\leq a<b<\infty\), the post-change
drift is
\[
\lambda\int_a^b
D_F(\widehat P_{1,t}^{\varepsilon}\|
    \widehat P_{\infty,t}^{\varepsilon})\,dt
=
\lambda\!\left[
\operatorname{KL}(\widehat P_{1,a}^{\varepsilon}\|
                  \widehat P_{\infty,a}^{\varepsilon})
-
\operatorname{KL}(\widehat P_{1,b}^{\varepsilon}\|
                  \widehat P_{\infty,b}^{\varepsilon})
\right].
\]
Thus truncation preserves a valid score statistic but replaces the full KL
information by its decrease over the chosen window.

\paragraph{Computational Cost and Local Mixture Approximations.}

Exact evaluation of both barycentric maps and local variances costs
\(O((N_0+N_1)d)\) per diffusion time, and hence
\(O(M(N_0+N_1)d)\) for \(M\) sampled times or quadrature nodes. When
\(h_t=\varepsilon+t\) is small or moderate, the softmax weights decay
exponentially in squared distance. One may therefore retain the \(K\) nearest
reference points, optionally augmented by \(L\) random points from the
remainder. This reduces cost but introduces pointwise error governed by the
omitted softmax mass and sample geometry; its sequential effect is not covered
by the present theory. At larger \(h_t\), the weights spread more broadly and
the retained neighborhood must grow.

\paragraph{Additional Barycentric Smoothing.}
For perturbations \(\xi_1,\ldots,\xi_M\) and \(\sigma\geq0\), one may replace
\(k_t^{r,\varepsilon}\) by
\[
\widetilde k_{t,\sigma}^{r,\varepsilon}(y)
=\frac1M\sum_{m=1}^M
k_t^{r,\varepsilon}(y+\sigma\xi_m),
\qquad
\widetilde s_{t,\sigma}^{r,\varepsilon}(y)
=\frac{\widetilde k_{t,\sigma}^{r,\varepsilon}(y)-y}{2h_t}.
\]
At \(\sigma=0\) this is the exact empirical diffusion score. For fixed
perturbations the smoothed field is conservative in the Gaussian-mixture
setting, but the induced surrogate densities do not generally follow a
common heat flow as \(t\) varies. The KL drift identity therefore applies
directly only to the unsmoothed construction.

\begin{table}[h]
\centering
\caption{Thresholds used in the anisotropic Gaussian simulation.}
\label{tab:simulation-thresholds}
\small
\setlength{\tabcolsep}{4pt}
\begin{tabular}{lccc}
\toprule
Method & 500 & 1000 & 2000\\
\midrule
CUSUM         & 4.3037 & 5.0234 & 5.8452\\
SCUSUM        & 3.9131 & 4.5986 & 5.3257\\
\methodname{}-quad  & 4.3267 & 5.0942 & 5.8760\\
\methodname{}-exact & 4.2891 & 5.0522 & 5.8115\\
\bottomrule
\end{tabular}
\end{table}

\subsection{Details for Proposition~\ref{prop:finite-window-information}}
\label{app:finite-window-information}

\begingroup

Applying Lemma~\ref{lem:heat-flow-kl} to
\((P_{1,t},P_{\infty,t})\) over \([0,T]\) gives
\[
\int_0^T D_F(P_{1,t}\|P_{\infty,t})\,dt
=\operatorname{KL}(P_1\|P_\infty)
-\operatorname{KL}(P_{1,T}\|P_{\infty,T}).
\]
This proves the first identity. Dividing by the positive initial KL
divergence gives the formula for \(R(T)\).

The function \(R(T)\) is a cumulative information curve. If it approaches
one quickly, a short diffusion window captures most of the forward KL
information; if it grows slowly, later diffusion times carry a non-negligible
fraction. The sampling density \(f\) controls how the integral is estimated
but cannot change this curve. This identity compares information targets; a
sequential delay comparison additionally requires equal false-alarm
calibration, as imposed in Remark~
\ref{rem:anisotropic-gaussian-separation} and Section~
\ref{sec:simulation}.
\endgroup

\subsection{Details for Remark~\ref{rem:anisotropic-gaussian-separation}}
\label{app:proof-anisotropic-gaussian-separation}

\begingroup

\paragraph{Derivation.}
Let $m=(\theta_0+\theta_1)/2$ and define
\[
A=\delta^\top\Sigma^{-2}\delta,
\qquad
B=\delta^\top\Sigma^{-3}\delta.
\]
The SCUSUM score difference is
$S_H(X;p_\infty)-S_H(X;p_1)=\delta^\top\Sigma^{-2}(X-m)$.
It has mean $-A/2$ and variance $B$ under $P_\infty$, and mean $A/2$
under $P_1$.  For the scaled increment $z=\lambda
\delta^\top\Sigma^{-2}(X-m)$,
\[
\mathbb E_\infty[e^{s z}]
=\exp\!\left(-\frac{s\lambda A}{2}
+\frac{s^2\lambda^2B}{2}\right).
\]
Its positive Cram\'er root is $A/(\lambda B)$.  Choosing the usual
unit-root normalization gives $\lambda=A/B$, so the post-change mean is
$I_{\mathrm{SCUSUM}}=\lambda A/2=A^2/(2B)$.
Let $\Sigma v_i=\sigma_i v_i$ and write
\[
\delta=\sum_{i=1}^d c_i v_i.
\]
Then, for $r\in\{1,2,3\}$,
\[
\delta^\top\Sigma^{-r}\delta
=\sum_{i=1}^d\frac{c_i^2}{\sigma_i^r}.
\]
Therefore,
\[
\frac{I_{\mathrm{SCUSUM}}}{I_{\mathrm{CUSUM}}}
=\frac{\left(\sum_i c_i^2/\sigma_i^2\right)^2}
{\left(\sum_i c_i^2/\sigma_i^3\right)
 \left(\sum_i c_i^2/\sigma_i\right)}.
\]
Apply Cauchy--Schwarz to
\[
a_i=\frac{|c_i|}{\sigma_i^{3/2}},
\qquad
b_i=\frac{|c_i|}{\sigma_i^{1/2}}.
\]
It gives
\begin{align*}
\left(\sum_i\frac{c_i^2}{\sigma_i^2}\right)^2
&=\left(\sum_i a_i b_i\right)^2 \\
&\leq
\left(\sum_i\frac{c_i^2}{\sigma_i^3}\right)
\left(\sum_i\frac{c_i^2}{\sigma_i}\right).
\end{align*}
Equality holds exactly when $1/\sigma_i$ is constant on the active set
\(\{i:c_i\neq0\}\).

For the ideal \methodname{} identity, set $C_t=\Sigma+2tI_d$ and
$m=(\theta_0+\theta_1)/2$. Twice the Hyv\"arinen-score difference at time
$t$ is
\[
2\left\{S_H(y;p_{\infty,t})-S_H(y;p_{1,t})\right\}
=2\delta^\top C_t^{-2}(y-m).
\]
Averaging over $Y=X+\sqrt{2t}\,W$ conditional on $X$ replaces $Y$ by $X$.
Since
\[
2\int_0^\infty(\Sigma+2tI_d)^{-2}\,dt=\Sigma^{-1},
\]
the fully integrated Rao--Blackwellized increment is
\[
\delta^\top\Sigma^{-1}(X-m)
=\log\frac{p_1(X)}{p_\infty(X)}.
\]
It therefore has the same information number as likelihood-ratio CUSUM.

\paragraph{Geometric Interpretation.}
In the eigenbasis of \(\Sigma\), likelihood-ratio CUSUM weights the component
\(c_i\) by \(1/\sigma_i\), whereas the zero-time score difference weights it
by \(1/\sigma_i^2\). The Cram\'er multiplier is scalar, so it can align these
weights only when all eigendirections carrying the shift have the same
variance. Along the heat flow, the directional weight becomes
\((\sigma_i+2t)^{-2}\). Its integral is \(1/(2\sigma_i)\), which restores the
inverse-variance geometry direction by direction. This is why ideal \methodname{}
matches likelihood-ratio CUSUM in this Gaussian model while SCUSUM can be
strictly less informative.
\endgroup

\subsection{Adaptive Weighting across Diffusion Times}
\label{app:weighted-discusum-details}

\begin{remark}[Adaptive weighting across diffusion times]
\label{rem:weighted-discusum}
Unlike the sampling density \(f\), which changes only estimator variance, an
offline linear combination of zero-time and Rao--Blackwellized diffusion
features changes the target evidence. For example, with fixed quadrature
times \(t_1,\ldots,t_L\), one may define
\[
z_n^{\mathrm{ad}}
=a_0\Delta_0(X_n)
+\sum_{\ell=1}^L a_\ell\overline G_{t_\ell}(X_n),
\qquad a_\ell\geq0,
\]
where \(\overline G_t\) is the Rao--Blackwellized contrast. If
\(\theta(a)>0\) solves
\(
\mathbb E_\infty[e^{\theta(a)z_n^{\mathrm{ad}}}]=1
\), then \(\theta(a)z_n^{\mathrm{ad}}\) has unit pre-change Cram\'er-root
normalization. The weights and normalization must be estimated and fixed
before monitoring; because they change the target integral, the KL guarantees
of Theorems~\ref{thm:false-alarm} and \ref{thm:detection-delay} do not follow
automatically. The construction below gives an example; its empirical evaluation is left to future work.
\end{remark}

\begingroup

Choose diffusion times \(t_1,\ldots,t_L>0\) and form
\[
\Phi(x)=\bigl(\phi_0(x),\overline G_{t_1}(x),\ldots,
\overline G_{t_L}(x)\bigr)^\top,
\qquad
\overline G_t(x)=\mathbb E_W[G_t(x,W)\mid X=x],
\]
where \(\phi_0\) is the zero-time SCUSUM evidence. Let
\(\widehat\mu_\infty\) and \(\widehat\mu_1\) be the feature means estimated
from the two reference models, and let \(\widehat\Sigma_\Phi\) be their pooled
covariance. For \(\rho>0\), define
\begin{equation}
\widehat\beta=(\widehat\Sigma_\Phi+\rho I)^{-1}
(\widehat\mu_1-\widehat\mu_\infty),
\qquad
\widehat b=-\frac12\widehat\beta^\top
(\widehat\mu_\infty+\widehat\mu_1).
\label{eq:weighted-beta}
\end{equation}
The raw evidence is
\(g_{\mathrm{AW}}(x)=\widehat\beta^\top\Phi(x)+\widehat b\). The Fisher
direction favors features with large estimated mean separation relative to
their covariance, while the intercept centers the fitted class means at
\[
\mp\frac12\widehat\beta^\top
(\widehat\mu_1-\widehat\mu_\infty).
\]
Thus the fitted pre- and post-change means have the intended signs whenever
the estimated feature means differ. By comparison, finite-window quadrature
fixes the coefficients of \(\overline G_{t_\ell}\) at their numerical
integration weights; it does not estimate a discriminant direction.

If the pre-change moment-generating function is finite near the origin and
admits \(\kappa_{\mathrm{AW}}>0\) satisfying
\[
\mathbb E_\infty
\!\left[e^{\kappa_{\mathrm{AW}}g_{\mathrm{AW}}(X)}\right]=1,
\]
then the unit-root-normalized increment and stopping rule are
\begin{align*}
z_n^{\mathrm{AW}}&=\kappa_{\mathrm{AW}}g_{\mathrm{AW}}(X_n),&
Z_n^{\mathrm{AW}}&=\max\{0,Z_{n-1}^{\mathrm{AW}}+z_n^{\mathrm{AW}}\},\\
T_{\mathrm{AW}}&=\inf\{n\geq1:Z_n^{\mathrm{AW}}\geq\tau\}.&&
\end{align*}
All feature weights, the intercept, and the Cram\'er root must be estimated
offline and fixed before monitoring. Because they are learned rather than
specified by the heat-flow integral, the resulting statistic requires its own
false-alarm calibration and does not automatically satisfy the KL drift
identities of \methodname{}.
\endgroup

\end{document}